\documentclass[10pt,journal,compsoc]{IEEEtran}
\usepackage{amsmath}
\usepackage{mathtools}
\usepackage{dsfont}
\usepackage{graphicx,psfrag,epsf} 
\usepackage{enumerate}
\usepackage{xcolor}
\usepackage{algorithm}
\usepackage{amsfonts}
\usepackage{algpseudocode}

\def\argmax{\mathop{\rm arg\,max}}%
\def\argmin{\mathop{\rm arg\,min}}%
\usepackage{subcaption}
\usepackage{multirow}
\usepackage{amsthm}
\usepackage[hyphens]{url}
\usepackage{hyperref}

\usepackage[T1]{fontenc}
\usepackage{adjustbox}

\newtheorem{theorem}{Theorem}

\newtheorem{remark}[theorem]{Remark}

\ifCLASSOPTIONcompsoc
  \usepackage[nocompress]{cite}
\else
  \usepackage{cite}
\fi

\begin{document}

\title{Rethinking Cost-sensitive Classification in Deep Learning via Adversarial Data Augmentation}

\author{Qiyuan Chen, Raed Al Kontar, Maher Nouiehed, Jessie Yang, Corey Lester
\IEEEcompsocitemizethanks{
\IEEEcompsocthanksitem Qiyuan Chen, Raed Al Kontar, Jessie Yang are with the Industrial \& Operation Engineering, University of Michigan - Ann Arbor. \protect
\IEEEcompsocthanksitem Maher Nouiehed is with Industrial Engineering, American University of Beirut.
\IEEEcompsocthanksitem Corey Lester is with the College of Pharmacy, University of Michigan - Ann Arbor.
}
}

\IEEEtitleabstractindextext{%
\begin{abstract}

Cost-sensitive classification is critical in applications where misclassification errors widely vary in cost. However, over-parameterization poses fundamental challenges to the cost-sensitive modeling of deep neural networks (DNNs). The ability of a DNN to fully interpolate a training dataset can render a DNN, evaluated purely on the training set, ineffective in distinguishing
a cost-sensitive solution from its overall accuracy maximization counterpart. This necessitates rethinking cost-sensitive classification in DNNs. To address this challenge, this paper proposes a cost-sensitive adversarial data augmentation (CSADA) framework to make over-parameterized models cost-sensitive. The overarching idea is to generate targeted adversarial examples that push the decision boundary in cost-aware directions. These targeted adversarial samples are generated by maximizing the probability of critical misclassifications and used to train a model with more conservative decisions on costly pairs. Experiments on well-known datasets and a pharmacy medication image (PMI) dataset made publicly available show that our method can effectively 
minimize the overall cost and reduce critical errors, while achieving comparable performance in terms of overall accuracy.
\end{abstract}

}

\maketitle

\IEEEdisplaynontitleabstractindextext

\IEEEpeerreviewmaketitle

\IEEEraisesectionheading{\section{Introduction}\label{intro}}
\IEEEPARstart{M}{ulti-class} classification is one of the most fundamental tasks in modern statistics and machine learning. It has seen many success stories across a wide variety of real-life applications such as computer vision \cite{dosovitskiy2020image, liu2021swin}, natural language processing \cite{devlin2018bert, brown2020language}, anomaly detection \cite{sun2022continual}, and fairness \cite{yue2021gifair}. 
Typical classification problems treat the cost of misclassifications equally and aim to maximize the overall accuracy in expectation across all classes and data points. However, in many practical settings, some misclassifications can be far more costly and critical compared to others - often accompanied by real-life safety implications. In the presence of such costs, models should be carefully designed to prevent critical errors. To contextualize this, consider the following real-life situation on medication errors that motivates this paper.   

Medication errors occur when patients take medications that are different in ingredient, strength, or dosage (e.g., oral tablet, oral capsule) from the medication prescribed for dispensing. According to the department of health and human services, these errors result in around 3 million outpatient medical appointments, 1 million emergency department visits, and 125,000 hospital admissions each year \cite{errors}.  
In general, such errors can expose the patient to dangerous side effects or result in a patient going untreated for their medical condition.  However, not all medication errors have the same consequences.  For example, a patient accidentally receiving glipizide, a medication for lowering blood sugar levels, instead of trazodone, a medication for sleep, could lower blood sugar too much, which, if ignored, can cause seizures and result in brain damage. In fact, it is more dangerous for a patient to receive glipizide instead of trazodone than for a patient to receive trazodone instead of glipizide as the consequences of taking glipizide are greater compared to trazodone \cite{ismp}. In this paper, we propose a cost-aware classification model that relies on generating targeted adversarial examples to be used in training a cost-sensitive model that avoids critical errors. We demonstrate the effectiveness of our proposed model by applying it to a pharmacy medication image (PMI) that we made publicly available at \url{https://deepblue.lib.umich.edu/data/concern/data_sets/6d56zw997 }.

The appearance of critical errors in multiclass classification is a challenge not unique to medication dispensing. Indeed, several papers have tried 
to address this challenge under the notion of cost-sensitive learning. Perhaps one of the earliest approaches dates back to 
\cite{breiman1984classification} that considered cost-sensitive classification via decision trees. They propose a model that uses decision trees to estimate the expected cost for each class and selects the label with the least cost. 
Since then, a wide variety of methods have been proposed for cost-sensitive learning across a wide variety of applications. A snapshot of the current literature is given in Sec. \ref{sec:literature}.

Yet, modern machine learning poses a fundamental challenge to traditional cost-sensitive learning, specifically when it comes to over-parameterized models such as deep neural networks (DNN). More specifically,  over-parameterized DNNs can fully interpolate (or over-fit) a training data set in practice. Indeed, many experiments \cite{zhang2016understanding, zhang2021understanding} have shown that DNNs can readily achieve zero cross-entropy training loss even on datasets that are completely unstructured random noise. Fortunately, despite over-fitting and in contrast to conventional statistical understanding, DNNs have exhibited outstanding generalization performance on unseen data \cite{devlin2018bert, brown2020language}. This phenomenon, known as benign over-fitting, has shown the potential benefits of over-parameterized modeling in recent years.

Despite their wide success in various applications, over-parameterized DNNs face a fundamental challenge in cost-sensitive classification tasks. If a model is clairvoyant, i.e., it can always reveal the underlying truth, then critical error costs will not affect the training as no misclassifications exist. 
This phenomenon motivates rethinking cost-sensitive classification in DNNs and reveals the need to go beyond training examples in cost-sensitive learning.

This paper tackles this exact problem. Specifically, we focus on general multiclass classification problems where costs are pairwise and possibly asymmetric. Under this setting, we propose a cost-sensitive adversarial data augmentation (CSADA) framework to render over-parameterized models cost-sensitive. The overarching idea is to generate targeted adversarial examples that push the decision boundary in cost-aware direction. Such focused adversarial samples 
are data perturbations that maximize the probability of getting critical misclassifications and are used to generate more conservative decisions on costly pairs. Experiments on multiple datasets, including the PMI dataset, show that our method can effectively minimize the overall cost and reduce critical errors, while achieving comparable performance in terms of overall accuracy. 
We note that although our main focus in this paper is DNNs, our approach can be directly incorporated within any classification approach to induce cost-sensitivity.

\textbf{Organization:} \quad The rest of the paper is organized as follows. We first start with a simple motivational example to fully contextualize the discussion above. Then in Sec.~\ref{sec:literature} we provide a brief overview of relevant related work. In Sec.~\ref{sec:model} we present our model and the training algorithm, followed by an illustrative proof-of-concept that sheds light on our model's intuition in Sec.~\ref{sec:concept}.  Experiments on various classification tasks and the PMI dataset are then presented in Sec.~\ref{sec:experiments}. Finally, Sec. \ref{sec:conclusion} concludes the paper with a brief discussion.

\subsection{A Simple Motivational Example} \label{sec:example}

We first start with a simple example highlighting the challenges faced by cost-sensitive modeling in DNNs. Consider a classification problem where given observable data $x\in \mathcal{X}$; we aim to predict the class label $y\in \mathcal{Y}$. Suppose the training data is $\mathcal{D}=\{x_i,y_i\}_{i=1\dots N}$, where $N$ denotes the total number of training points. Now, consider the setting where the costs of misclassifications are not equal among all the pairs. Let matrix $\mathcal{C}$ be the cost matrix with non-negative entries $c(y,z)$ where for a given class $y\in \mathcal{Y}$, a cost $c(y,z)$ is incurred when a model predicts $z \in {\cal Y}$ for a data point with true label $y$. Notice that $c(y,z)$ need not be equal to $c(z,y)$ (recall the medical example). Moreover, assume that a correct classification incurs zero cost so that all elements on the diagonal of $\mathcal{C}$ are zero. 

In a typical classification task, we aim to learn a parameterized model that predicts the label of an input $x$. This is typically done by estimating the probability that a data sample $x$ returns a label $z \in {\cal Y}$, which we denote by $p_{z}(\theta;x)$, and selecting the label with the highest probability. 
In a conventional cost-insensitive setting, the goal is to maximize the probability of returning the true label. 
\begin{align} \label{eq:simple}
    \max_{\theta}~\mathds{E}_{x,y} \left[p_{y}(\theta;x) \right]  \approx \frac{1}{N}\cdot\sum_{i=1}^N \left[p_{y_i}(\theta;x_{i}) \right] \, ,
\end{align}

\noindent where $p_{y_i}(\theta;x_{i})$ is the correct prediction probability. Notice that if we simply add a natural logarithm to the predicted probabilities, then \eqref{eq:simple} becomes the well-known cross-entropy loss that is the most common loss function used in DNN classification tasks. 

Now, for cost-sensitive settings, our aim is to minimize the overall cost. A natural extension of \eqref{eq:simple} is given by
\begin{align} \label{eq:cost} \nonumber
    \min_{\theta}~&\mathds{E}_{x,y} \left[\sum_{z\in \mathcal{Y}}c(y,z)\cdot p_z(\theta; x)\right]  \\ 
    \approx~\min_{\theta}~ & \frac{1}{N}\cdot\sum_{i=1}^N \left[\sum_{z\in \mathcal{Y}}c(y_i,z)\cdot p_z(\theta; x_i) \right] \, ,
\end{align}

\noindent where $c(\cdot,\cdot)$ are non-negative costs of misclassifications. Note that if $c(y,z)=1$ for all $y, \, z \in {\cal Y}$, then~\eqref{eq:cost} retrieves~\eqref{eq:simple}. Moreover, one can easily show that both objectives~\eqref{eq:simple} and~\eqref{eq:cost} achieve the optimal solution if and only if $p_{y_i}(\theta;x_{i})=1$ or equivalently 
$p_{z}(\theta;x_{i})=0$ for all $\{z \neq y_i; z \in \mathcal{Y}\}$. 
Hence, if the functional space is unconstrained or the model capacity is large enough to fully interpolate the training examples, solving \eqref{eq:cost} becomes equivalent to solving \eqref{eq:simple}. 


When considering an over-parameterized setting, it is often the case that there exists a $\theta$ such that $p_{y_i}(\theta;x_{i})\approx1$ for all the training data points in $\mathcal{D}$ \cite{zhang2016understanding,arpit2017closer,belkin2018overfitting}. Therefore, the cost matrix $\mathcal{C}$ will not affect the models as no misclassifications occur, and the cost is always zero on the training dataset. More specifically, if the model is able to perfectly fit training data, then any loss function evaluated purely on the training set becomes ineffective in distinguishing a cost-sensitive solution (in \eqref{eq:cost}) from its overall-accuracy maximization counterpart (in \eqref{eq:simple}). 

\begin{figure}[htbp!]
  \centering
  \begin{subfigure}[b]{\linewidth}
  \centering
    \includegraphics[width=0.65\linewidth]{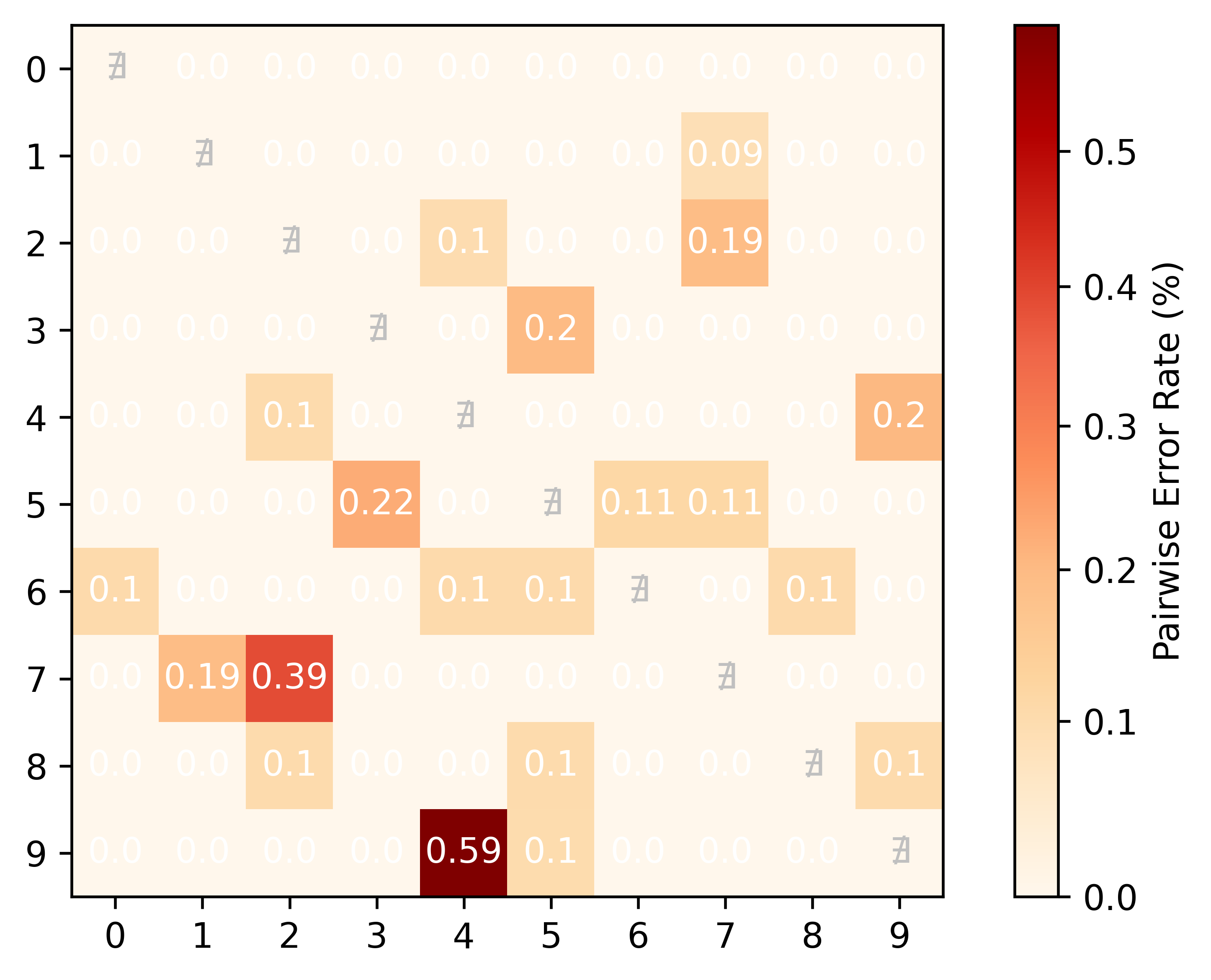}
    \caption{Baseline model}
    \label{fig:MNIST_Before1}
  \end{subfigure}
  \begin{subfigure}[b]{\linewidth}
  \centering
    \includegraphics[width=0.65\linewidth]{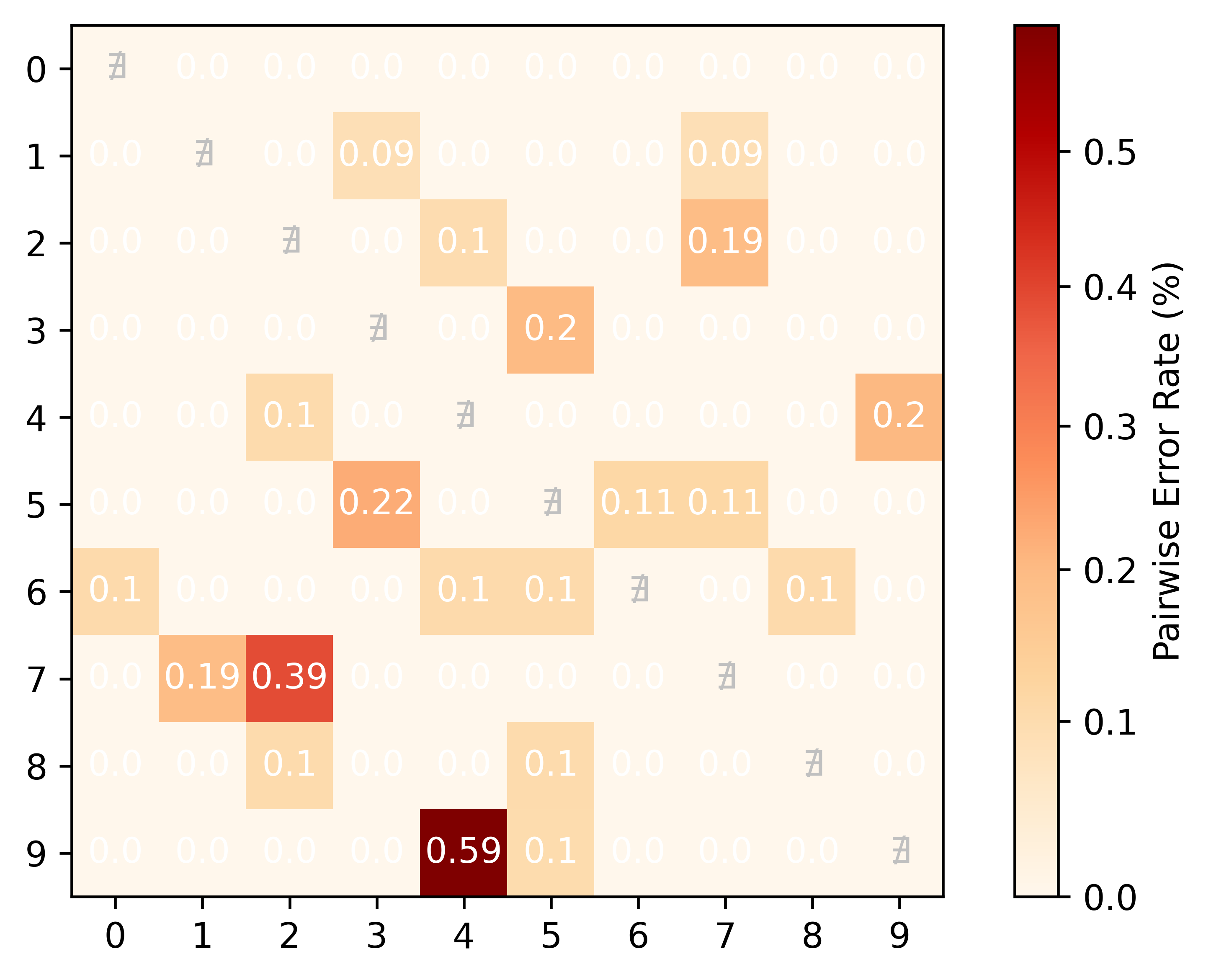}
    \caption{Cost-sensitive model}
    \label{fig:MNIST_Penalty2}
  \end{subfigure}
  \caption{Similarity of Baseline and  cost-sensitive model}
  \label{fig:motivation}
\end{figure}

Next, we provide a numerical example highlighting the failure of cost-sensitive learning in DNNs. We consider the popular MNIST dataset comprising of handwritten images of digits 0-9 \cite{deng2012mnist} trained on ResNet-34. Using a cross-entropy loss, we train a vanilla Baseline model. Details on experimental settings can be found in Sec. \ref{sec:experiments} where we revisit this experiment. For the Baseline model, at the end of the last epoch, the model reaches $100\%$ accuracy on the training set. Fig. \ref{fig:MNIST_Before1} shows the pairwise error rate across the ten classes.

Now starting with the pre-trained model and for every possible pair $(y',z')$ between the ten classes, we train a cost-sensitive DNN where only predicting $y'$ as $z'$ incurs a cost of one while all the other error costs are set to zero. We then define an extreme case of the cross-entropy loss where the sole goal is to prevent the model from making critical errors. Starting from the pre-trained model, the objective is given as
\begin{align}\label{eq:extreme}
   \min_{\theta} \frac{1}{N}\cdot\sum_{i=1}^N \left[-\mathds{1}_{y_i=y'}\cdot\log (1 - p_{z'}(\theta;x_i))\right] .
\end{align}

The purpose of this test is to observe if this strong cost-sensitive penalty can bias an overfitted model to reduce critical errors. The results for all 90 experiments are shown in Fig. \ref{fig:MNIST_Penalty2}. In the figure, each entry of the matrix corresponds to the pairwise error rate of the experiment done, with this pair being the critical one. Indeed, the results confirm the previous discussion. Even with an extremely biased objective aiming only to minimize the critical error rate, we see no improvement over the Baseline. This again confirms that in over-parametrized models, cost-sensitive training can be very challenging if we completely depend on a training set that can be perfectly fit. Further illustrations along this line can be seen in Sec. \ref{sec:experiments}. Our proposed method in Sec. \ref{sec:model} goes beyond the training set by generating adversarial examples that are specifically targeted to induce cost sensitivity.

\section{Literature Overview} \label{sec:literature}
\subsection{Cost-Sensitive Learning}
Settings where misclassification errors have different costs have been studied under the umbrella of cost-sensitive learning. Most existing research is on traditional models like logistic regression \cite{bahnsen2014example}, decision trees \cite{lomax2013survey}, and support vector machines \cite{iranmehr2019cost}. For neural networks (NNs), cost-sensitive NNs were first studied by \cite{kukar1998cost}. In their work, a cost-weighted empirical risk function was introduced where misclassification probabilities are multiplied by their corresponding costs. Later on, \cite{zhou2005training} extended this approach to NN with imbalanced datasets where methods such as oversampling and undersampling classes were introduced to induce cost-sensitivity. The authors show that their methods were effective in shallow NN. Cost-sensitive Deep Neural Network (CSDNN) \cite{chung2015cost} is one of the earliest works that develop effective cost-sensitive methods for DNNs by incorporating cost into the softmax layer. The performance of this approach was challenged by \cite{chung2020cost} when the model is an over-parameterized DNN. 

In general, cost-sensitive literature, including the work mentioned above, can be decomposed into three categories. The first category makes no change during the training phase but incorporates cost during inference/prediction. Mainly after predictive probability vectors of all classes are attained, they are multiplied with a cost matrix to get the expected cost for each class. The class with the least cost is selected as the cost-sensitive prediction \cite{breiman1984classification,domingos1999metacost,zadrozny2001learning}. These methods can be incorporated into any cost-sensitive training approach. However, they rely on the accuracy of the predicted probability vectors, which are usually over-confident in over-parameterized NNs \cite{guo2017calibration}. The second category makes model-specific modifications \cite{tu2010one, chung2015cost}. For example, \cite{tu2010one} modifies the SVM model to predict a cost vector instead of probabilities. Following a similar idea, \cite{chung2015cost} proposes CSDNN and $\text{CNN}_{\text{SOSR}}$ which incorporate a regression layer to the NN to predict costs. One limitation of this category is that algorithms are often specialized and cannot be generalized to other models. For instance, CSDNN uses a special neural network that consists of fully connected layers and sigmoid activation. This limits its applications in other NN structures such as ResNet \cite{he2016deep} or Transformers \cite{vaswani2017attention}. The third category modifies the training data distribution. As aforementioned, the most widely adopted approach in this category is replacing the empirical risk with a cost-weighted empirical risk \cite{kukar1998cost, zadrozny2003cost, chan1999toward,zhou2005training}. However, as explained in Sec. \ref{sec:example}, re-weighting empirical risk is ineffective in over-fitted models. 


The method proposed in this paper is closest to the third category, but instead of re-weighting data points, we incorporate targeted adversarial examples that can effectively move the decision boundaries of a trained DNN in cost-aware directions in the presence of overfitting. Being a data augmentation method, our framework can be applied to most DNN structures as well as models beyond NNs.


\subsection{Adversarial Attacks and Adversarial training}
While DNNs have outstanding performance on modern machine learning tasks, they are known for their fragility to adversarial attacks \cite{szegedy2013intriguing}. One can adversarially add an imperceptible perturbation to the input data and significantly change outputs. Such perturbed data are called adversarial examples. Gradient-based attacks are commonly used in generating adversarial examples. Common methods include Fast Gradient Sign Method (FGSM) \cite{goodfellow2014explaining},  DeepFool \cite{moosavi2016deepfool}, and Projected Gradient Descent (PGD) \cite{kurakin2016adversarial}. Interestingly, research has shown that over-parameterization leads to higher sensitivity to adversarial attacks in linear models \cite{donhauser2021interpolation}, and random feature regression \cite{hassani2022curse}. To address this issue, several methods use the generated adversarial examples in training to improve models' robustness \cite{Volpi2018unseen}. Another interesting property observed by \cite{cao2017mitigating} is that adversarial examples are located near the decision boundaries of NNs. Based on this observation, \cite{heo2019knowledge} uses adversarial examples to depict the decision boundaries of NNs and uses the adversarial examples as boundary supporting samples for knowledge distillation. 

Recent literature has seen work on cost-sensitive robustness against adversarial attacks  \cite{zhang2018cost, shen2022adversarial}. While robustness in literature refers to surviving an adversarial attack (e.g., FGSM, PGD), cost-sensitive robustness aims at improving robustness for costly pairs. Our work tackles a different problem where there is no adversary in the inference stage. Rather, we propose targeted data augmentation to prevent critical errors on future \textit{natural} data. We also point out that there exist methods that use the generative adversarial network (GAN) for data augmentation, often with a focus on imbalanced datasets \cite{sampath2021survey}. Yet, GAN generates synthetic examples to mimic the natural training examples, which are then added to classes with few data to re-balance the datasets. Our adversarial augmentation scheme aims to generate targeted adversaries that push decision boundaries rather than replicate natural data.

\section{Model Development} \label{sec:model}

Our problem can be viewed through the lens of adversarial training of statistical and machine learning models. Unlike typical settings, the adversary in our model aims at generating targeted adversarial examples that increase the chance of having critical errors. The defender, however, aims at finding a cost-aware robust model that avoids such errors. In simple words, our proposed model generates targeted adversarial examples to train a cost-aware robust model against critical errors.


Consider a classification model $f(\theta, \cdot)$ parameterized by $\theta$ and let $p_z(\theta; \, (x_i, \,y_i))$ be the probability of classifying point $x_i$ with class $z$. For a given pair $(y, \, z)$, we solve a maximization objective that finds a perturbation of the input data that maximizes the $z$-th class predicted probability as follows:
\begin{equation}\label{max_model}
  \delta^{(y,z)} = 
\argmax_{\|\delta\| \leq \epsilon} p_z\left(\theta;\left(x+\delta,y\right)\right),  
\end{equation}
\noindent where $\|\cdot\|$ is a distance measure, and $\epsilon$ is the threshold that limits the magnitude of the perturbation. This is repeated for each data point $(x_i, y_i)$ and for each $z \in {\cal Y}$. The purpose of these maximization problems is to find targeted attacks on a sample input data $x_i$ aiming for directions that favor critical classes. 
Our overreaching goal is to find data samples that fools the model to make costly errors. These samples are then used in training to create cost-aware decision boundaries that reduce critical misclassification errors. Mathematically speaking, we propose the following penalized cost-aware bi-level optimization formulation of the problem
\begin{align}\label{bi-level formulation}
\min_{\theta}~ &\ell_{augmented}(\theta;x,y,\delta) \nonumber\\=\min_{\theta}~&\nonumber \frac{1}{N}\sum_{i=1}^N~\Big[\ell\left(f(\theta,x_i),y_i\right)\\
&+\lambda\sum_{z \in \mathcal{Y}}\Tilde{c}(y_i,z)\ell\left( f\left(\theta,x_i+\delta^{(y_i,z)}\right),y_i\right)\Big] \\
\text{s.t.}& \quad
\delta^{(y_i,z)} = 
\argmax_{\|\delta\|_p \leq \epsilon} p_z\left(\theta;\left(x_i+\delta,y_{i}\right)\right)\nonumber
\end{align}
\noindent where $\ell(\cdot)$ is the {\it cross-entropy loss}, $f(\theta, x_i)$ is the latent representation of the input $x_i$, and $\theta$ is a vector of model parameters. Our formulation penalizes the objective with a cost-aware weighted sum of critical misclassification errors. The first term of our objective is a typical empirical risk objective that measures the sum of losses between the true labels $\{y_i\}_{i=1}^N$ and the corresponding model outputs $\{f(\theta, x_i)\}_{i=1}^N$. The second term, however, is a penalty term that penalizes missclassifications of augmented examples according to their corresponding weights. In particular, the second term is the weighted sum of the losses incurred by the examples generated by the maximization problems. The weights $\Tilde{c}(y_i,z)$ are defied by $c(y_i,z)^\tau/\sum_{y,z\in |\mathcal{Y}|}{c(y,z)^\tau}$, where $\tau$ is a temperature value that controls the emphasis on critical costs. A higher temperature will help filter out smaller costs when there are many classes. The higher the cost of a misclassification error, the higher the penalty coefficient of that term. Moreover, $\lambda$ is a hyper-parameter that regulates the trade-off between the the first term that corresponds to the regular empirical loss and the penalty term which penalizes critical misclassification errors for perturbed data. To summarize, our maximization problem generates $|{\cal Y}|$ adversarial perturbations for each data point targeting all classes $z \in {\cal Y}$. These examples are further used in the minimization problem for finding optimal model parameters for a cost-sensitive penalized empirical risk formulation that penalizes critical errors.

\subsection{Relation to Adversarial Training}
Typical adversarial training problems are formulated as min-max optimization problems that search for optimal model parameters when using worst-case data perturbations. In particular, the adversary in typical adversarial training aims at finding adversarial examples that maximize the loss objective, i.e., data perturbations that fool the model the most. The defender, however, aims at finding a robust model that minimizes the loss when using the perturbed data.

Unlike the typical setting, 
our model generates and exploits targeted adversarial examples that increase the chance of having critical errors. This major difference results in distinct objectives for the max and min optimization problems. We next show that, in binary classification settings, the bi-level optimization problem in~\eqref{bi-level formulation} can be formulated as a min-max non-zero-sum game.

\begin{theorem}\label{lm: Min-Max Formulation}
Consider the objective defined in~\eqref{bi-level formulation} with binary labels $\{y, \bar{y} = 1-y\}$ and $\ell(\cdot)$ being the cross-entropy loss. Then solving ~\eqref{bi-level formulation} is equivalent to solving the following min-max problem
\begin{align}\label{eq:min-max formulation}
  \min_{\theta} \max_{\|\delta\|_{\infty} \leq \epsilon} \,& \ell_{augmented}(\theta;x,y,\delta) \nonumber\\
  =  \min_{\theta} \max_{\|\delta\|_{\infty} \leq \epsilon} \,& \frac{1}{N}\sum_{i=1}^N~\left[\ell\left(f(\theta,x_i),y_i\right)\right. \nonumber\\
  &\left.+ \lambda \Tilde{c}(y_i,\bar{y}_i) \ell\left( f\left(\theta,x_i+\delta_i\right),y_i\right)\right]  
\end{align}
\noindent where $\delta = \left( \delta_{i} \right)$  for all $i \in \{1, \ldots ,N\}$.
\end{theorem}

\begin{proof}
Since the functions are disjoint, we first obtain
\begin{align*}
&\max_{\|\delta\|_{\infty} \leq \epsilon} \;\; \sum_{i=1}^N \Tilde{c}(y_i,\bar{y}_i)\ell\left( f\left(\theta,x_i+\delta_i\right),y_i\right) \\
= &\sum_{i=1}^N \Tilde{c}(y_i,\bar{y}_i) \max_{\|\delta\|_{\infty} \leq \epsilon}\ell\left( f\left(\theta,x_i+\delta_i\right),y_i\right). 
\end{align*}
To complete the proof, it suffices to show that
\begin{align*}
&\argmax_{\|\delta_i\| \leq \epsilon} \;\ell\left( f\left(\theta,x_i+\delta_i\right),y_i\right) \\
=& \argmax_{\|\delta_i\| \leq \epsilon} \; p_{\bar{y}_i} \left(\theta; \left(x_i+\delta_i\right),y_i \right).
\end{align*}
In binary classification settings, the softmax layer is simply a sigmoid function, which we denote by $\sigma (\cdot)$. Let $r_{y_i}\left(\theta; \left(x_i + \delta_i \right) \right)$ be the latent representation corresponding to class $y_i$ of the data point $x_i + \delta_i$ before the sigmoid function. With $\ell(\cdot)$ being the cross-entropy loss, we obtain
\begin{align*}
\ell\left( f\left(\theta,x_i+\delta_i\right),y_i\right)& = -\log \; \sigma\left( r_{y_i}\left(\theta; \left(x_i + \delta_i \right) \right)\right) \\
&= -\log \; \left( 1 - \sigma\left(- r_{y_i}\left(\theta; \left(x_i + \delta_i \right) \right)\right)\right)\\
&= -\log \left(1- \; \sigma\left( r_{\bar{y}_i}\left(\theta; \left(x_i + \delta_i \right) \right)\right) \right),
\end{align*}
where the last equality holds due the binary setting assumption. Knowing that $-\log(\cdot)$ is a monotonically decreasing function, we get
\begin{align*} 
&\argmax_{\|\delta_i\| \leq \epsilon} \;\ell\left( f\left(\theta,x_i+\delta_i\right),y_i\right) \\
=~& \argmin_{\|\delta_i\| \leq \epsilon} \;\left[1- \sigma\left( r_{\bar{y}_i}\left(\theta; \left(x_i + \delta_i \right) \right)\right) \right]\\
=~& \argmax_{\|\delta_i\| \leq \epsilon}  \;\sigma\left( r_{\bar{y}_i}\left(\theta; \left(x_i + \delta_i \right) \right)\right)\\
=~&  \argmax_{\|\delta_i\| \leq \epsilon} \;p_{\bar{y}_i}\left(\theta; \left(x_i+\delta_i\right),y_i \right),
\end{align*}
which completes the proof.

\end{proof}

\begin{remark}
In the binary case, maximizing the probability of the untrue label is equivalent to minimizing the probability of the true label, which translates to maximizing the cross-entropy loss. The last implication results from the monotone nature of the $\log$ function. This relation allows us to formulate our problem as a min-max non-zero-sum game. The formulation presented in Theorem~\ref{lm: Min-Max Formulation} can be seen as a cost-aware adversarial objective that further penalizes errors that cost more. This model can be applied in settings where {\it false positive} and {\it false negative} errors have different costs (e.g., medication dispensing example in Sec. \ref{intro}).
\end{remark}

The result presented in Theorem~\ref{lm: Min-Max Formulation} fails to hold in non-binary classification settings. A simple counter-example is presented for a 3-label classification problem. Consider a data point $(x,y)$ for which
\begin{align*}
&r_{y}(\theta; (x + \delta)) = 3\delta,\\
&r_{z}(\theta; (x + \delta)) = 2\delta, \\
\mbox{ and }\; &r_{3}(\theta; (x + \delta)) = -10\delta,
\end{align*}
where $r_{3}(\theta; (x + \delta))$ is the latent representation for the third class. Then,
\begin{align*}
&\argmax_{\|\delta\|\leq \epsilon}\ell\left( f\left(\theta,x+\delta \right),y\right) \\
=~& \argmax_{\|\delta\|\leq \epsilon} -\log \; \sigma\left( r_{y}\left(\theta; \left(x + \delta \right) \right)\right)
\end{align*}
is maximized by setting $\delta = -\epsilon$. However,
\[\argmax_{\|\delta\| \leq \epsilon} \; p_{z} \left(\theta; \left(x+\delta\right),y \right) >0.\]

This counter-example shows that generating targeted examples according to the maximization problem in~\eqref{bi-level formulation} is not equivalent to maximizing the loss function as used in typical adversarial training. This is mainly due to the nature of our problem, which aims at generating targeted adversarial examples instead of worst-case perturbations. Despite having distinct objectives, our formulation can be seen through the lens of non-zero-sum games. In the next section, we discuss the algorithm proposed to solve the objective presented in~\eqref{bi-level formulation}.

\subsection{Multi-step Gradient Descent Ascent with Rejection}

Recent years have witnessed extensive research for solving min-max problems arising from adversarial settings. One of the first algorithms developed for adversarial training is the Fast Gradient Sign Method (FGSM) proposed by Goodfellow et al. \cite{goodfellow2014explaining}. This algorithm can be seen as a single-step method that constructs an adversarial example by solving a linearized approximation of the inner maximization problem. A multi-step variant of FGSM, denoted by $K$-PGD, that applies $K$ gradient ascent steps for solving the maximization problem was proposed in \cite{madry2018towards}. Several variants of this multi-step ascent descent algorithm have been widely deployed in saddle point formulations. Despite having no convergence guarantees in non-convex settings, such approaches have shown wide empirical success. In fact, no algorithm is known to converge in general non-convex non-concave optimization problems. In this paper, we use a variant of the multi-step ascent descent method for solving the problem defined in~\eqref{bi-level formulation}. More specifically, our algorithm performs multiple ascent steps to generate targeted perturbed adversarial examples and then performs a gradient descent step on the cost-aware penalized objective.

Our overreaching goal is to induce critical errors in training by generating data samples that are close to the decision boundary between the corresponding labels. This idea of generating samples close to the decision boundary was inspired by \cite{moosavi2016deepfool}. As illustrated in Sec.~\ref{sec:concept}, having data points near the boundary is particularly important in over-fitted models. We deploy this in our algorithm by imposing a rejection mechanism that rejects steps that generate examples with a label different from the true or targeted labels. Moreover, when the attack succeeds, i.e., the perturbed data sample is classified as $z$, which is our target, the algorithm returns this perturbed sample. In our setting, one justification for stopping after successful attacks is that successful attacks suffice to push the decision boundaries of over-parameterized models. Our algorithm is detailed in Algorithm~\ref{alg:max}.

\begin{algorithm}[H]
\caption{Projected Gradient Ascent with Rejection} \label{alg:max}
\begin{algorithmic}[1]
\State \textbf{Input:} Data point $\{x, y\}$, direction $z$, and weights $\theta$
\State Initialize ${\delta}^{(0)} \gets 0$
\If{pre-attack prediction $f(x+{\delta}^{(k-1)},\theta)$ is not $y$}
\State \textbf{Output:} $\delta^{(0)}$
\EndIf
\For{$k=1,\dots,K$}
\State $\Delta^{(k)} = \eta_2 \cdot \nabla_{\delta}\log(p_z(\theta;(x+{\delta}^{(k-1)},y)))$
\State ${\delta}^{(k)} \gets \text{Proj}_{\mathbb{B}(0,\epsilon)}\left[{\delta}^{(k-1)} + \Delta^{(k)} \right]$

\If{post-attack prediction $f(x+{\delta}^{(k)},\theta) \notin \{y,z\}$}
\State \textbf{Output:} ${\delta}^{(k-1)}$
\ElsIf{post-attack prediction $f(x+{\delta}^{(k)},\theta)$ is $z$}
\State \textbf{Output:} ${\delta}^{(k)}$
\EndIf
\EndFor
\State \textbf{Output:} ${\delta}^{(K)}$
\end{algorithmic}
\end{algorithm}

Our algorithm starts by checking whether, for a given model parameter $\theta$, the model predicts the true label. This is determined in Step 3 of the algorithm, which guarantees that only data samples that are correctly predicted are modified. For such data samples, our algorithm iteratively performs projected gradient ascent steps as long as the model is predicting the true label. If the predicted label is neither $y$ nor $z$, the new step is rejected, and the previous iterate is returned. Moreover, when the model predicts $z$, the ascent steps are terminated, and the current iterate is returned. Otherwise, the model continues performing projected ascent steps up to $K$ iterations. The final adversarial examples are fed along with the original data to perform a gradient descent step on the penalized cost-aware objective. The complete algorithm is detailed in Algorithm~\ref{alg:min}. In the next section, we describe our stochastic version of the proposed algorithm.

\begin{algorithm}[H]
\caption{Multi-step Gradient Decent-Ascent}\label{alg:min}
\begin{algorithmic}[1]
\State \textbf{Input:} Natural data set $\mathcal{D}$, cost matrix $\mathcal{C}$
\State \textbf{Output:} Trained weights $\theta^{(T)}$
\State Initialize $\theta^{(0)}$ as pretrained weights on $\ell(\theta;x,y)$
\For{$t=1,\dots,T$}
\For{$z\in \mathcal{Y}$}
\State ${\delta}_{i}^{(y_i,z)} \gets$ \textbf{Algorithm \ref{alg:max}}$(x_{i},y_{i},z,\theta^{(t-1)})$
\EndFor
\State $\theta^{(t)} \gets \theta^{(t-1)} - \eta_1 \cdot \nabla_{\theta} \ell_{augmented}(\theta^{(t-1)};x,y,\delta)$
\EndFor
\end{algorithmic}
\end{algorithm}

\subsection{Stochastic Multi-Ascent Descent with Rejection}
One important challenge for our proposed algorithm is the computational complexity incurred by generating $(|{\cal Y}|-1)$ adversarial examples for each data sample. This becomes computationally intractable when the number of data samples or the number of classes is large. In this section, we propose a stochastic version of the algorithm that samples at each iteration a batch of data and {\it one} critical pair which are then used to compute an unbiased estimate of the gradient of the main objective. More specifically, our proposed method randomly samples a batch from the training data, denoted by ${\cal B}$, and samples a critical pair from a categorical distribution. The word stochastic often implies using a batch of samples at every iteration instead of the whole dataset. The word stochasticity here also implies sampling a critical pair for every selected batch. This reduces the number of generated adversarial examples in each iteration from $(|{\cal Y}|-1)\cdot N$ to one.  We define our critical pair sampling distribution according to the cost of the pairs. More specifically, we choose the probability of selecting a specific pair $(y_{\mathcal{B}},z_{\mathcal{B}}) \in (\mathcal{Y},\mathcal{Y})$ to be the normalized cost $\Tilde{c}(y_{\mathcal{B}},z_{\mathcal{B}})$. For a given batch ${\cal B}$ and its corresponding critical pair selection $(y_{\mathcal{B}},z_{\mathcal{B}})$, our stochastic objective can be defined as
\begin{align*}
\ell_{stochastic}(\theta;x,y,\delta_{{\cal B}})=\frac{1}{|\mathcal{B}|}&\sum_{i\in \mathcal{B}}~\Big[\ell\left(\theta;(x_{i},y_{i})\right) \\
& + \lambda\ell\left(\theta;\left(x_i+\delta_{{\cal B}}^{(y_i,z_{\mathcal{B}})},y_{i}\right)\right)\Big],
\end{align*}
where
\[
\delta_{{\cal B}}^{(y_i, z_{{\cal B}})}= 
\begin{cases}
\begin{array}{ll}
 \displaystyle\argmax_{\|\delta\| \leq \epsilon} \, p_{z_{{\cal B}}} \left(\theta; (x_i + \delta, y_i) \right) & \mbox{if } y_i = y_{{\cal B}}\\
 0 & \mbox{otherwise}.
\end{array}
\end{cases}
\]

One can easily notice that the gradient of $\ell_{stochastic}$ is an unbiased estimator of the gradient of $\ell_{augmented}$ which is defined in~\eqref{bi-level formulation}. The detailed algorithm proposed for solving the stochastic version of our model is detailed in Algorithm~\ref{alg:stoch_min}.

\begin{algorithm}[H]
\caption{Multi-step Stochastic Gradient Decent-Ascent}\label{alg:stoch_min}
\begin{algorithmic}[1]
\State \textbf{Input:} Natural data set $\mathcal{D}$, cost matrix $\mathcal{C}$
\State \textbf{Output:} Trained weights $\theta^{(T)}$
\State Initialize $\theta^{(0)}$ as pretrained weights on $\ell(\theta;x,y)$
\For{$t=1,\dots,T$}
\State Sample a mini-batch $\mathcal{B}\subseteq [N]$
\State Sample one pair $(y_{\mathcal{B}},z_{\mathcal{B}})$
\For{$i$ in $\mathcal{B}$}
\If{$y_i=y_{{\cal B}}$} 
\State$\delta_{{\cal B}}^{(y_i,z_{\mathcal{B}})} \gets$ \textbf{Algorithm \ref{alg:max}}$(x_{i},y_i,z_{\mathcal{B}},\theta^{(t-1)})$
\EndIf
\EndFor
\State $\theta^{(t)} \gets \theta^{(t-1)} - \eta_1 \cdot \nabla_{\theta} \ell_{stochastic}(\theta^{(t-1)};x,y,\delta_{{\cal B}})$
\EndFor
\end{algorithmic}
\end{algorithm}

\section{Proof of Concept} \label{sec:concept}
\begin{figure*}[htb]
  \centering
  \begin{subfigure}[b]{0.26\linewidth}
    \includegraphics[width=\linewidth]{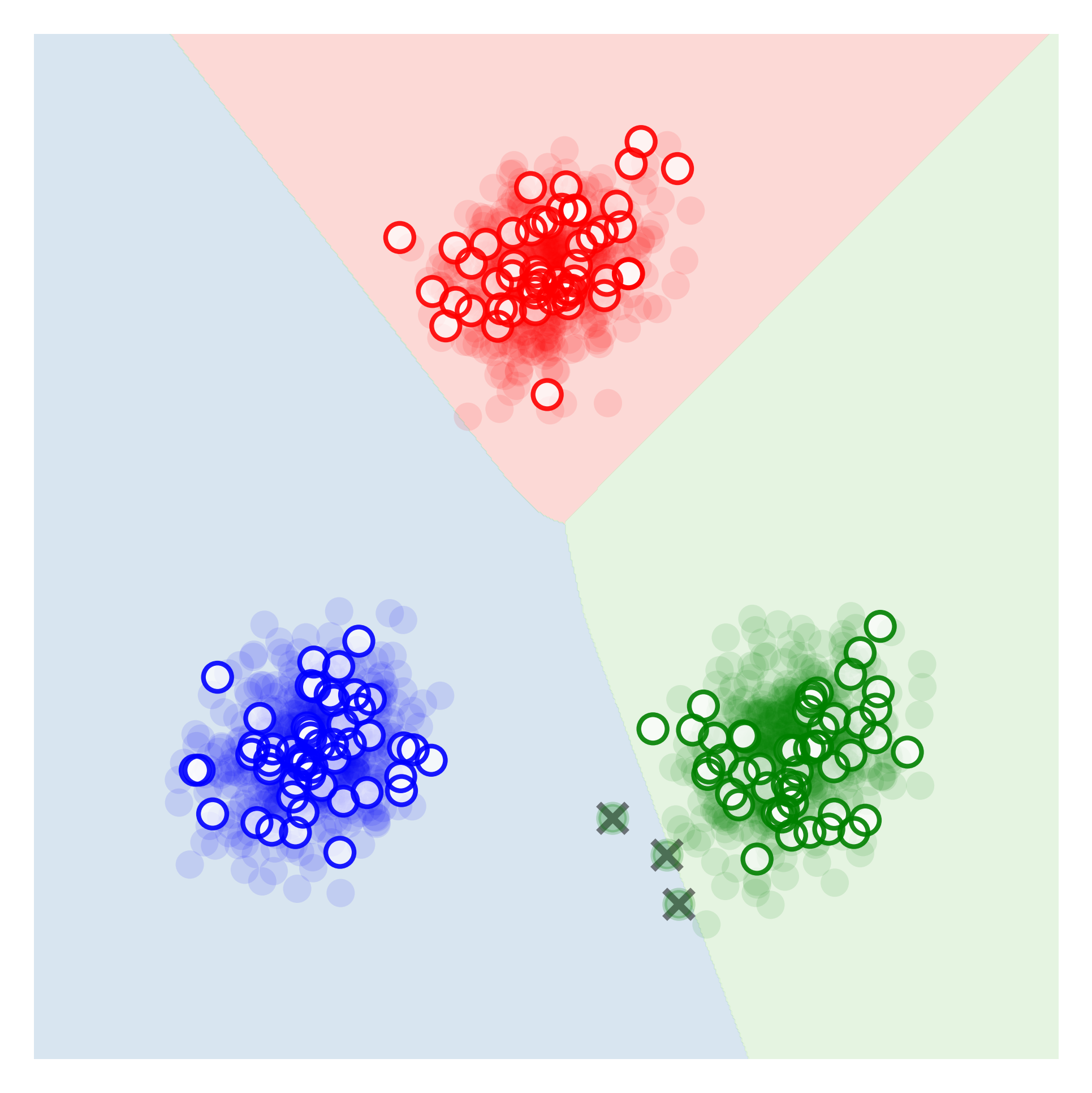}
    \caption{Critical errors (in crosses)}
    \label{fig:toy1}
  \end{subfigure}
  \begin{subfigure}[b]{0.26\linewidth}
    \includegraphics[width=\linewidth]{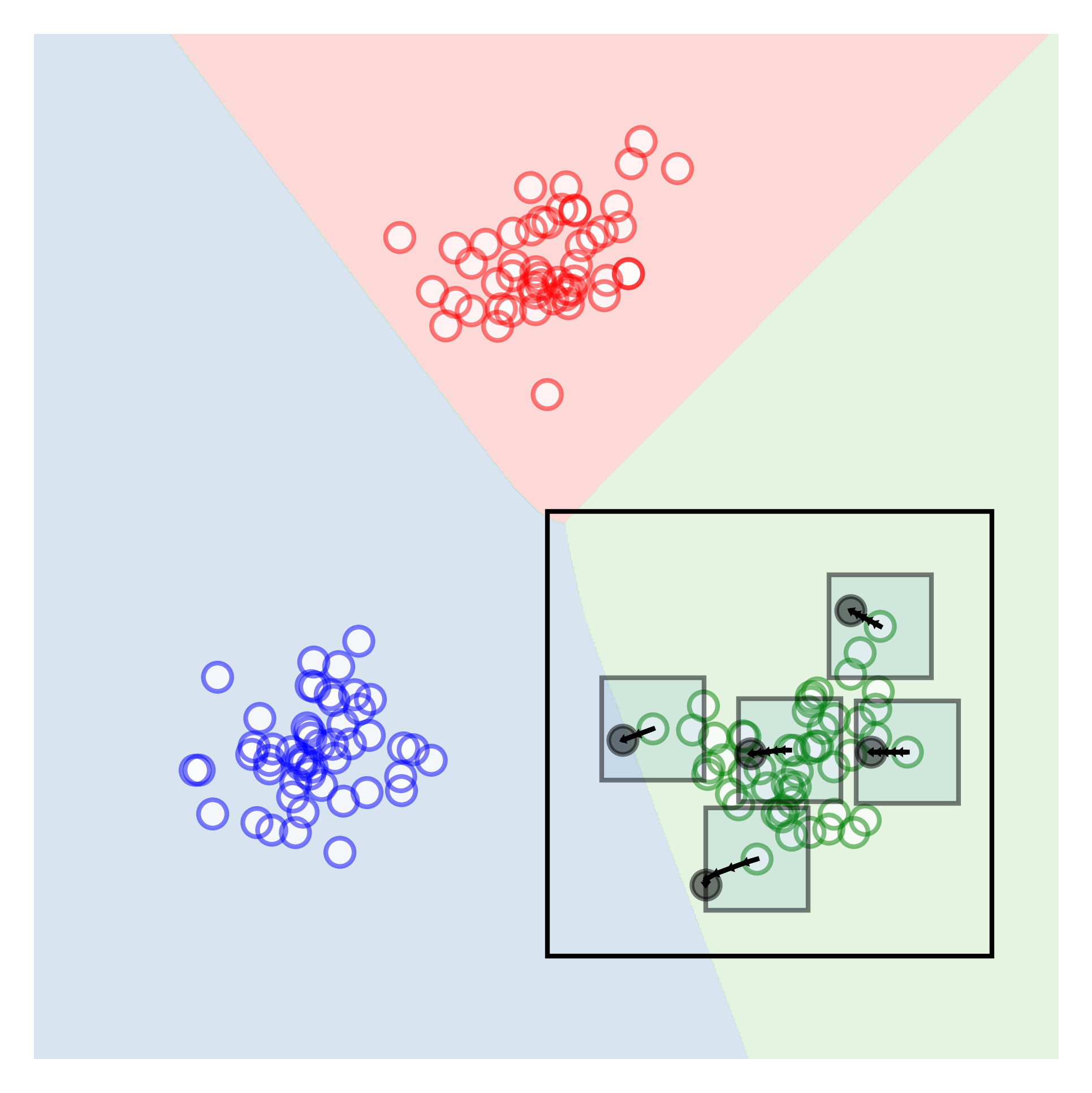}
    \caption{Adversarial examples (before)}
    \label{fig:toy2}
  \end{subfigure}
    \begin{subfigure}[b]{0.26\linewidth}
    \includegraphics[width=\linewidth]{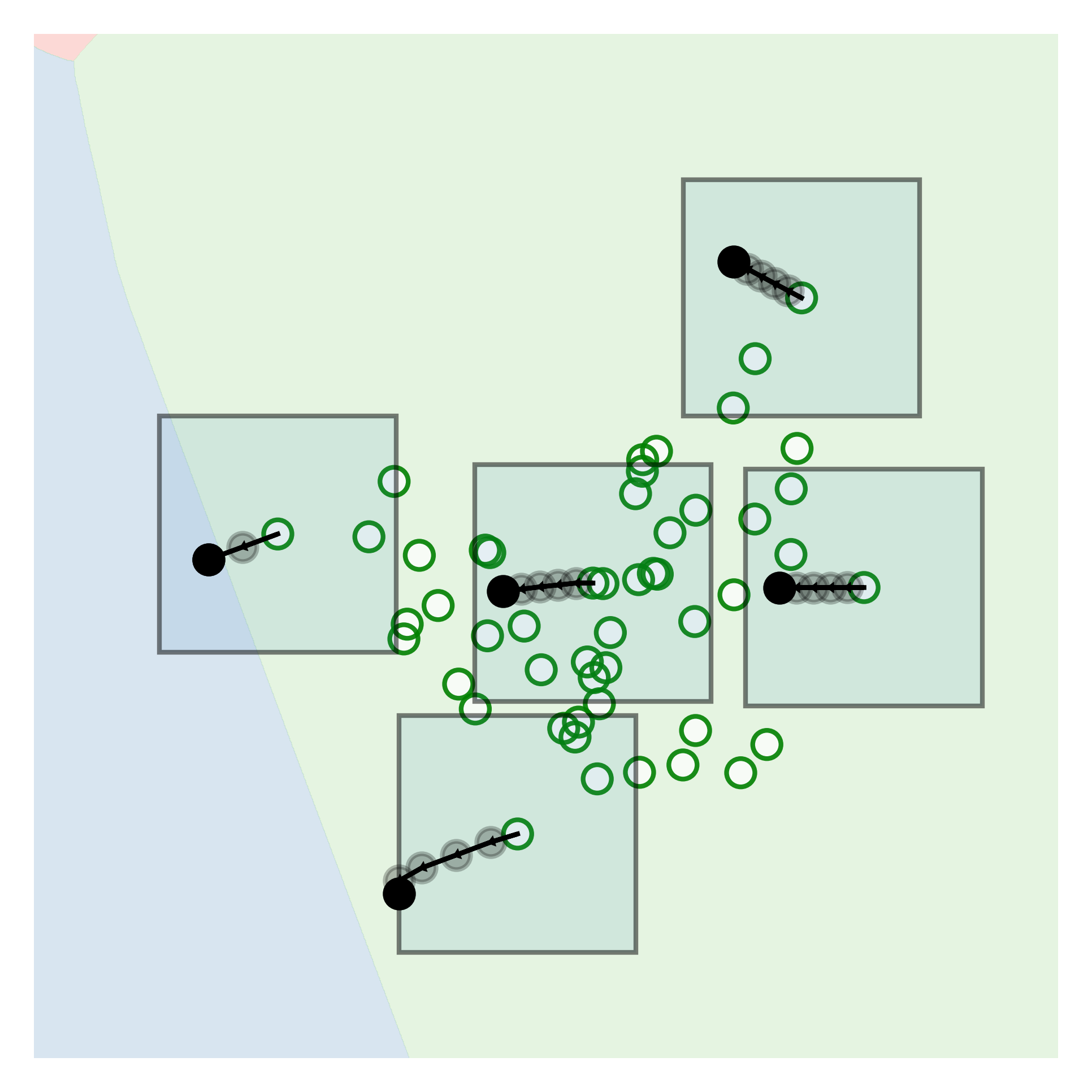}
    \caption{Enlarged Fig. \ref{fig:toy2}}
    \label{fig:toy3}
  \end{subfigure}
    \begin{subfigure}[b]{0.26\linewidth}
    \includegraphics[width=\linewidth]{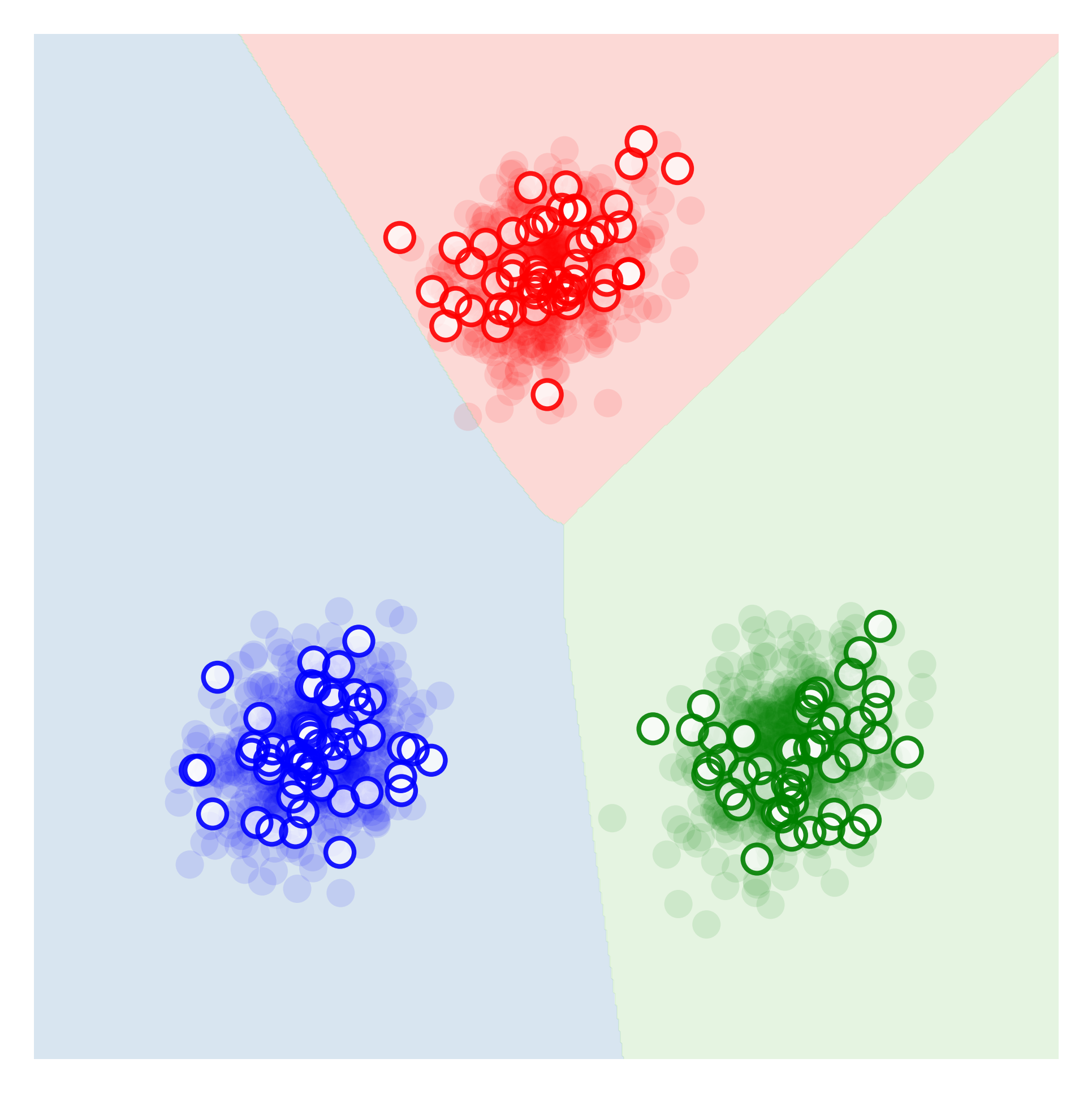}
    \caption{Avoided critical errors}
    \label{fig:toy4}
  \end{subfigure}
    \begin{subfigure}[b]{0.26\linewidth}
    \includegraphics[width=\linewidth]{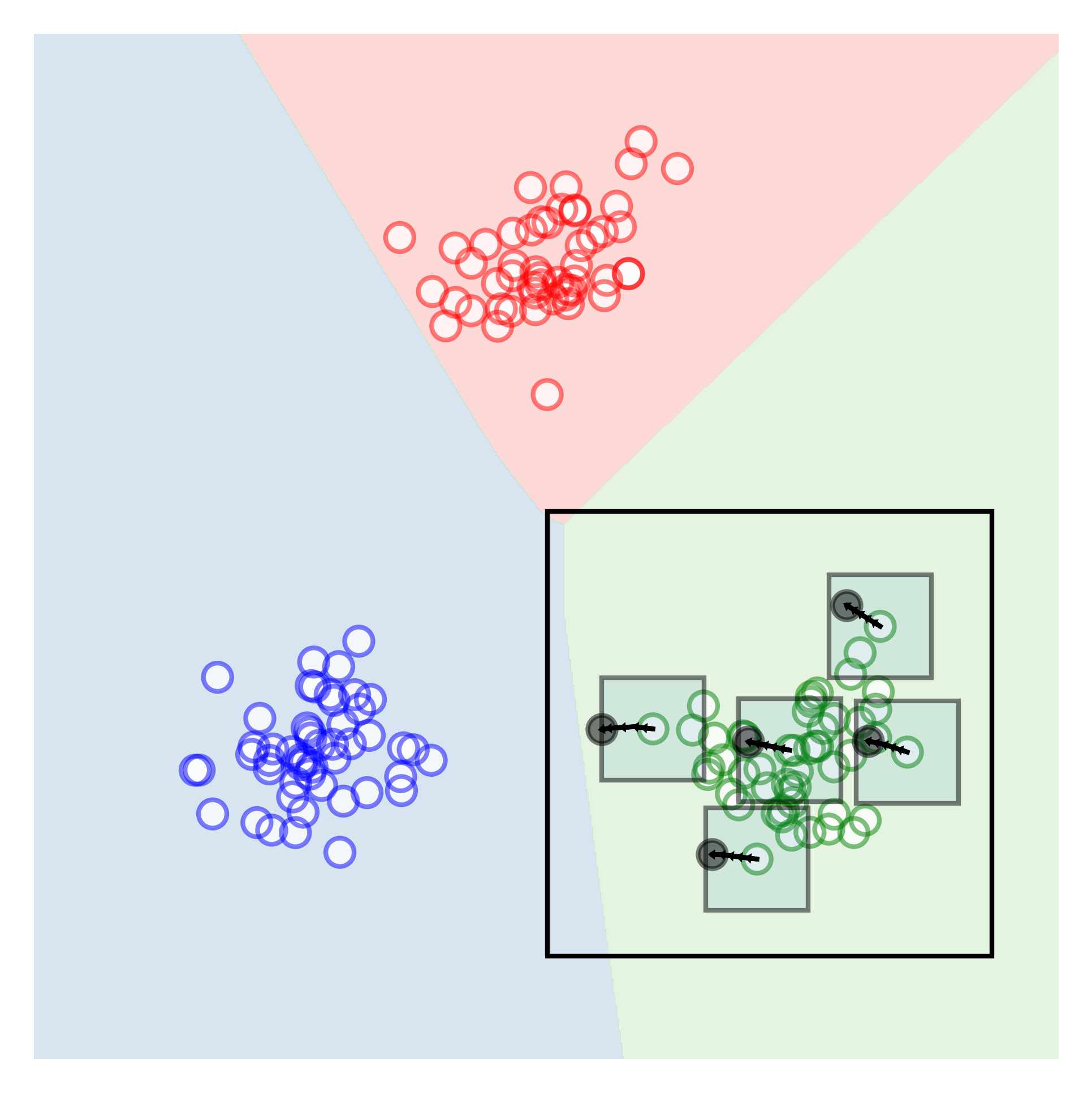}
    \caption{Adversarial examples (after)}
    \label{fig:toy5}
  \end{subfigure}
    \begin{subfigure}[b]{0.26\linewidth}
    \includegraphics[width=\linewidth]{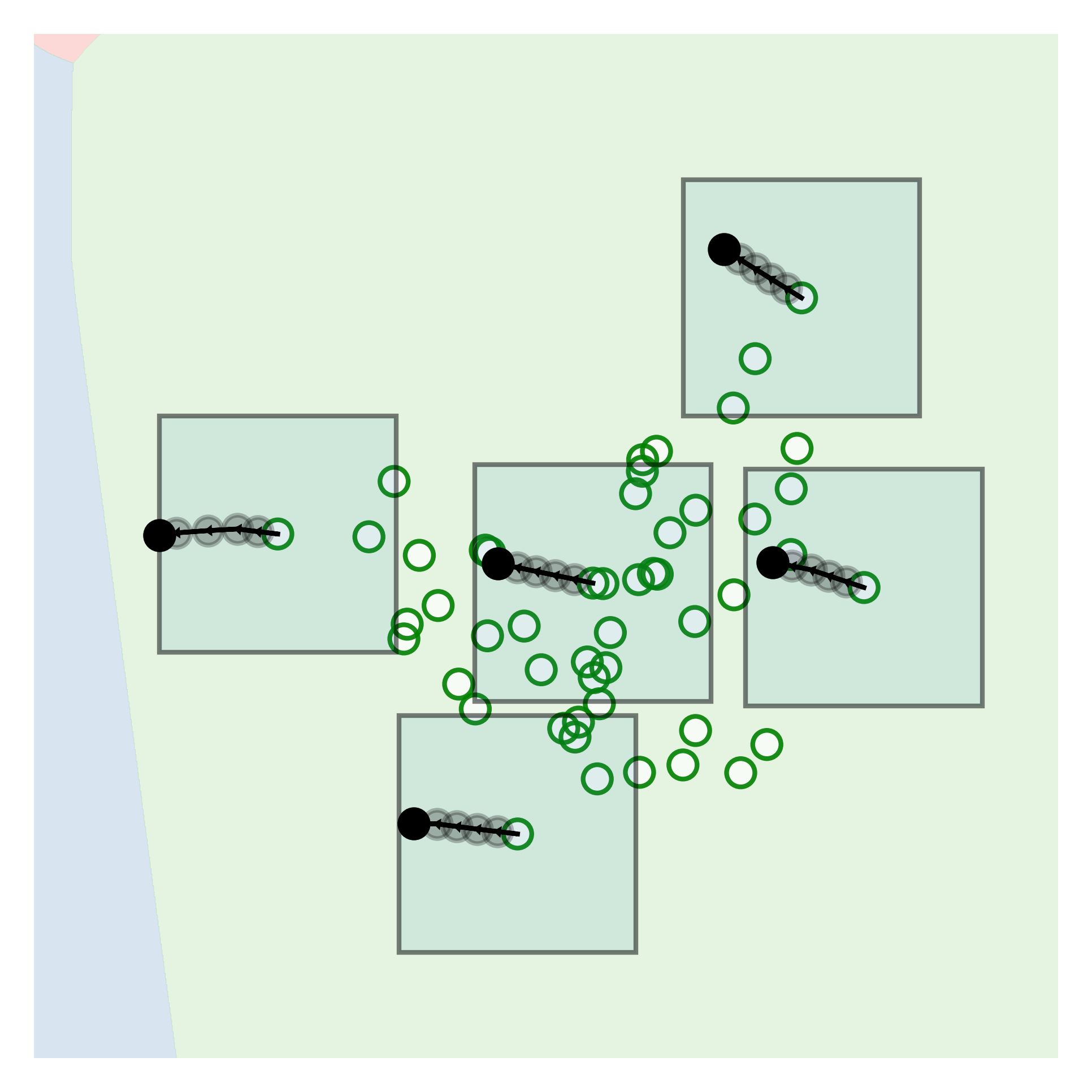}
    \caption{Enlarged Fig. \ref{fig:toy5}}
    \label{fig:toy6}
  \end{subfigure}
  \caption{Before (\ref{fig:toy1},\ref{fig:toy2},\ref{fig:toy3}) and After (\ref{fig:toy4},\ref{fig:toy5},\ref{fig:toy6}) CSADA}
  \label{fig:toy7}
\end{figure*}

A toy classification problem is presented in this section to demonstrate our model's ability to prevent critical errors. Three classes are generated from independent two-dimensional Gaussian distributions, labeled in red, blue, and green (see Table \ref{table:Toy_Data} and Fig. \ref{fig:toy1}). We define the cost from green to blue to be one and all others to be zero (see Table \ref{table:Toy_Cost}) and set $\tau=1$. As such, only misclassifying green with blue incurs a cost. For each class, 50 points (shown in hollow circles) are sampled for training, and 500 points are sampled for testing (samples shown in solid translucent circles). For this classification task, we train a multi-layer perception (MLP) with 50 nodes. We first optimize the vanilla MLP without augmentation using gradient descent. At convergence, the decision boundaries are shown in Fig.  \ref{fig:toy1}. One can directly realize that the model successfully reaches $100\%$ accuracy on the training samples. However, the decision boundary between the green and the blue class leads to multiple critical errors on the test set (labeled in green crosses).

\begin{table}[htb]
\centering
\resizebox{\columnwidth}{!}{%
\begin{tabular}{cccccc}
\hline
Class & Mean                & Covariance  & Train Samples & Test Samples \\ \hline
Red & $[0,8]^\top$ & \multirow{3}{*}{$\begin{bmatrix} 2&0.5 \\ 0.5&2 \end{bmatrix}$}  & 50 & 500 \\
Green & $[7,-6]^\top$  &           & 50            & 500         \\
Blue  & $[-7,-6]^\top$ &           & 50            & 500         \\ \hline
\end{tabular}
}
\caption{Toy Example Data Setting}
\label{table:Toy_Data}
\end{table}

\begin{table}[htb]
\centering
\begin{tabular}{c|ccc}
\hline
      & Red & Green & Blue \\ \hline
Red   & 0   & 0     & 0    \\
Green & 0   & 0     & 1    \\
Blue  & 0   & 0     & 0    \\ \hline
\end{tabular}
\caption{Toy Example Cost Setting}
\label{table:Toy_Cost}
\end{table}

Using the approach detailed in Algorithm \ref{alg:min}, we show in Fig. \ref{fig:toy2} the trajectories of the adversarial attacks on five different points. The large box in black is enlarged in Fig. \ref{fig:toy3} for better visualization. These trajectories highlight the idea behind our augmentation scheme. The generated adversarial examples pushed the green data points towards the boundary between green and blue regions. The grey points are the intermediate gradient ascent steps, whereas the black points are the outputted adversarial examples. The adversarial attacks either stop when the points go across the classification boundary or when the designated number of maximization steps is reached. Boxes surrounding each point denotes the norm constraint $\|\delta\| \leq \epsilon$. Here we use infinity norm,  $\epsilon=1.5$, $K=5$, and $\eta_2=0.05$.

The results of our trained model are shown in Fig. \ref{fig:toy4}. After incorporating adversarial samples, the decision boundary between the green and blue was shifted towards the blue class, while the other boundaries were mostly unaffected. Interestingly, our updated model does not have any critical errors on the test set. As shown in Figs.\ref{fig:toy5} and \ref{fig:toy6}, at convergence, the boundary between green and blue is far enough from the green class so that no more successful attacks (attacks that go across the boundary) can be achieved within the given number of maximization steps (i.e., budget). 

\section{Experiments} \label{sec:experiments}
\subsection{The Failure of Simple Reweighting}
\begin{figure*}[htb]
  \centering
  \begin{subfigure}[b]{0.3\linewidth}
    \includegraphics[width=\linewidth]{figures/MNIST_Baseline.png}
    \caption{Baseline model}
    \label{fig:MNIST_Before}
  \end{subfigure}
  \begin{subfigure}[b]{0.3\linewidth}
    \includegraphics[width=\linewidth]{figures/MNIST_Penalty.png}
    \caption{Penalty method}
    \label{fig:MNIST_Penalty}
  \end{subfigure}
  \begin{subfigure}[b]{0.3\linewidth}
    \includegraphics[width=\linewidth]{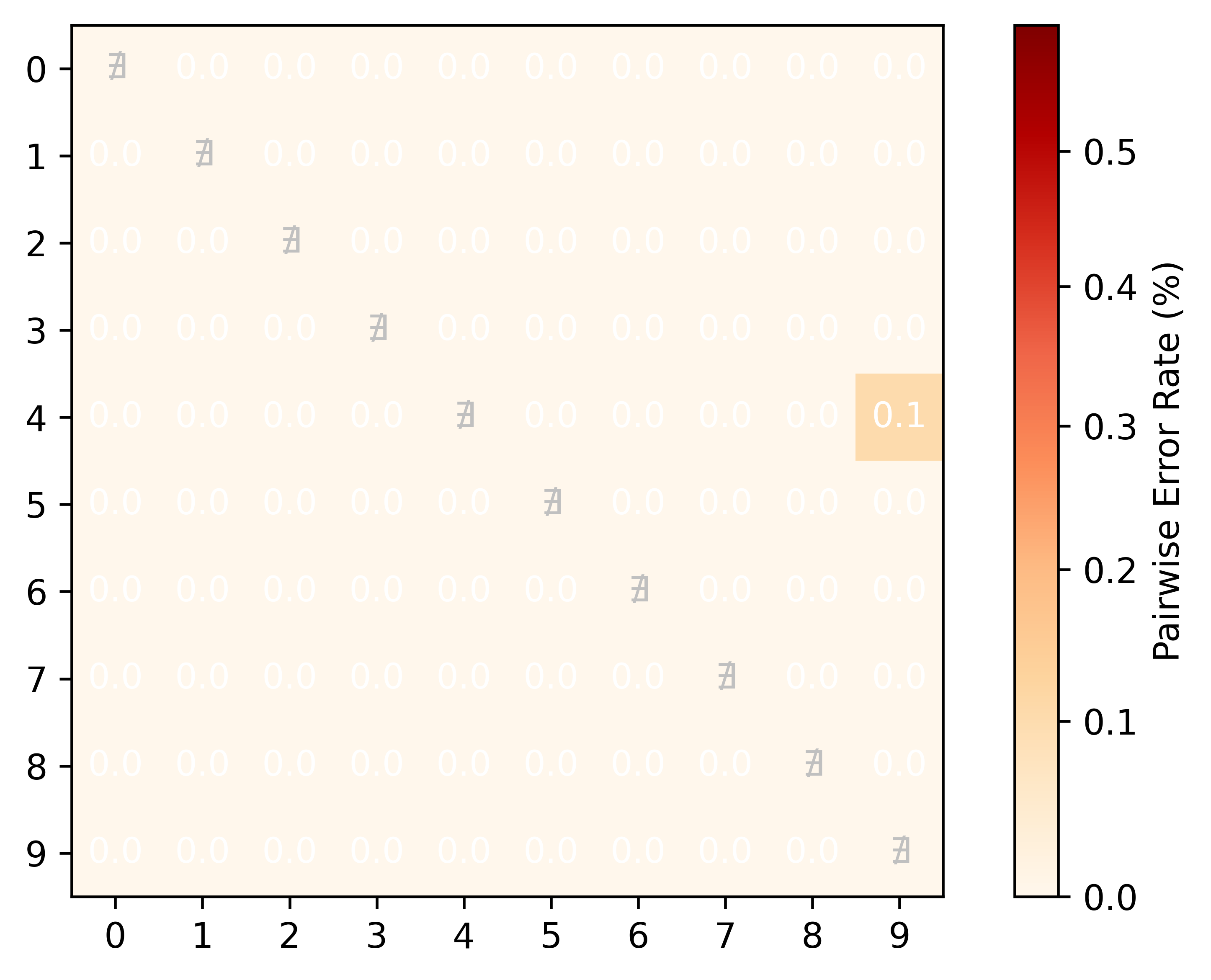}
    \caption{CSADA}
    \label{fig:MNIST_After}
  \end{subfigure}
  \caption{Comparison of Baseline, Penalty, CSADA models on MNIST}
  \label{fig:MNIST}
\end{figure*}
\begin{figure*}[htb]
  \centering
  \begin{subfigure}[b]{0.32\linewidth}
    \includegraphics[width=\linewidth]{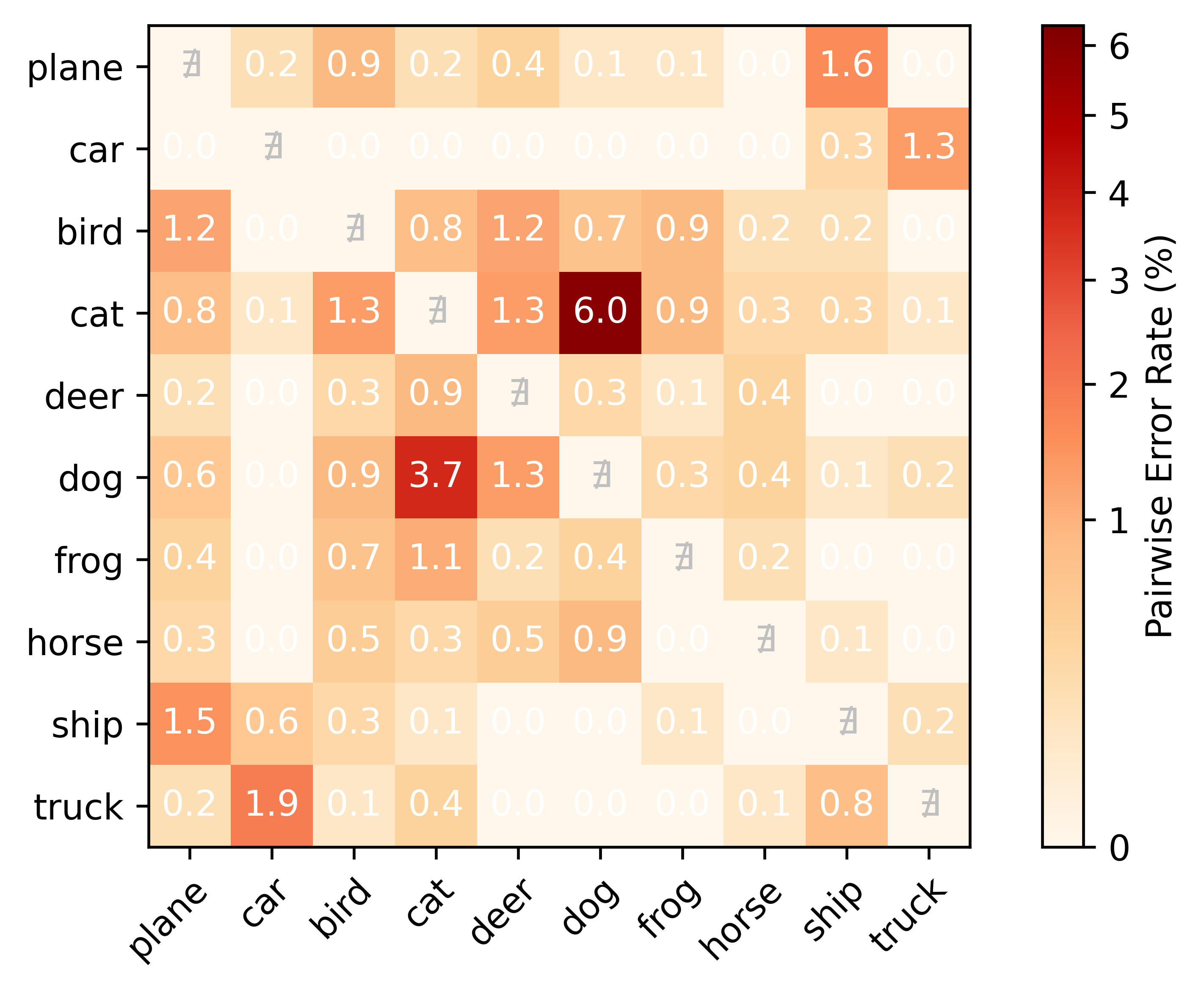}
    \caption{Baseline model}
    \label{fig:CIFAR_Before}
  \end{subfigure}
  \begin{subfigure}[b]{0.32\linewidth}
    \includegraphics[width=\linewidth]{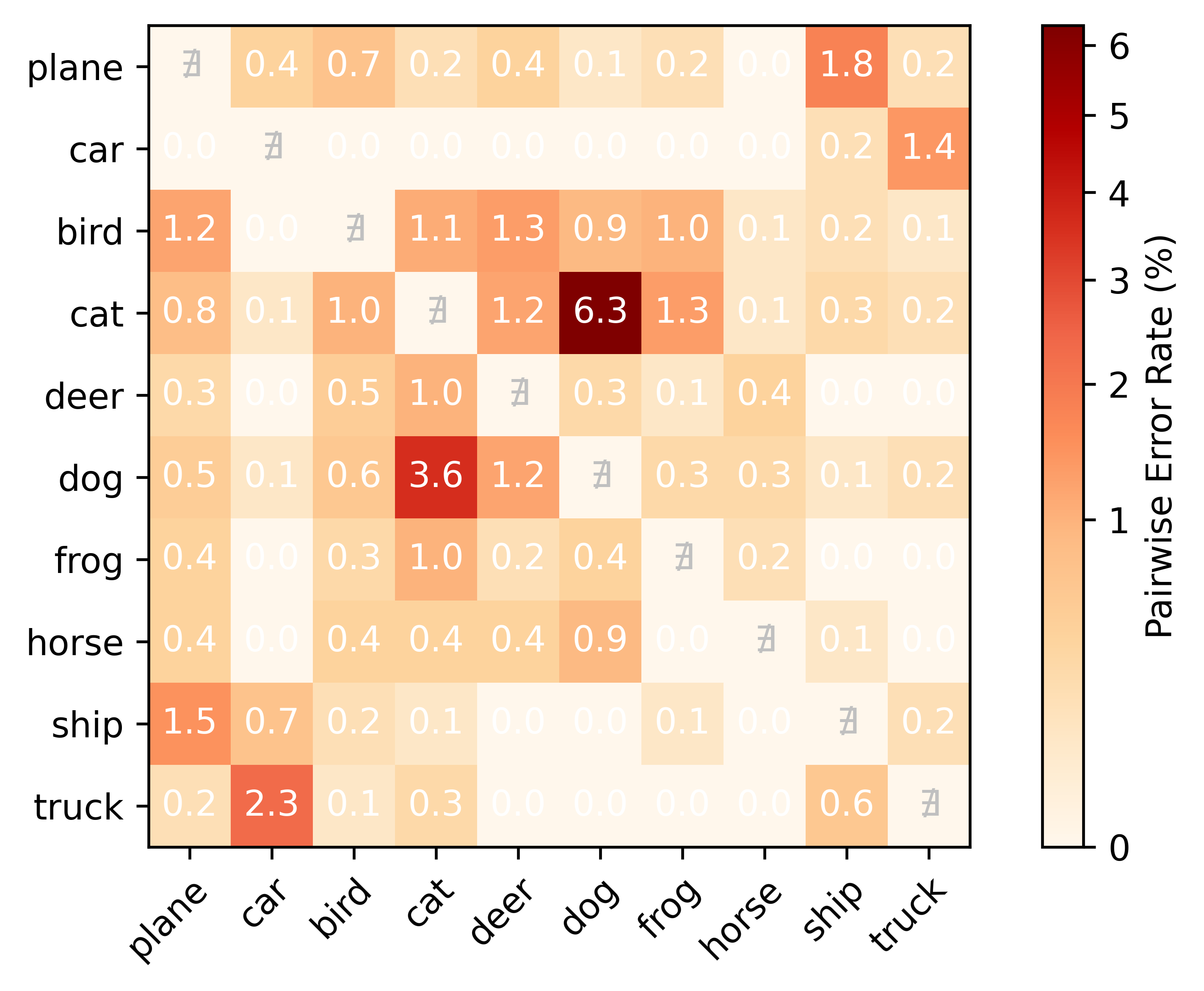}
    \caption{Penalty method}
    \label{fig:CIFAR_Penalty}
  \end{subfigure}
  \begin{subfigure}[b]{0.32\linewidth}
    \includegraphics[width=\linewidth]{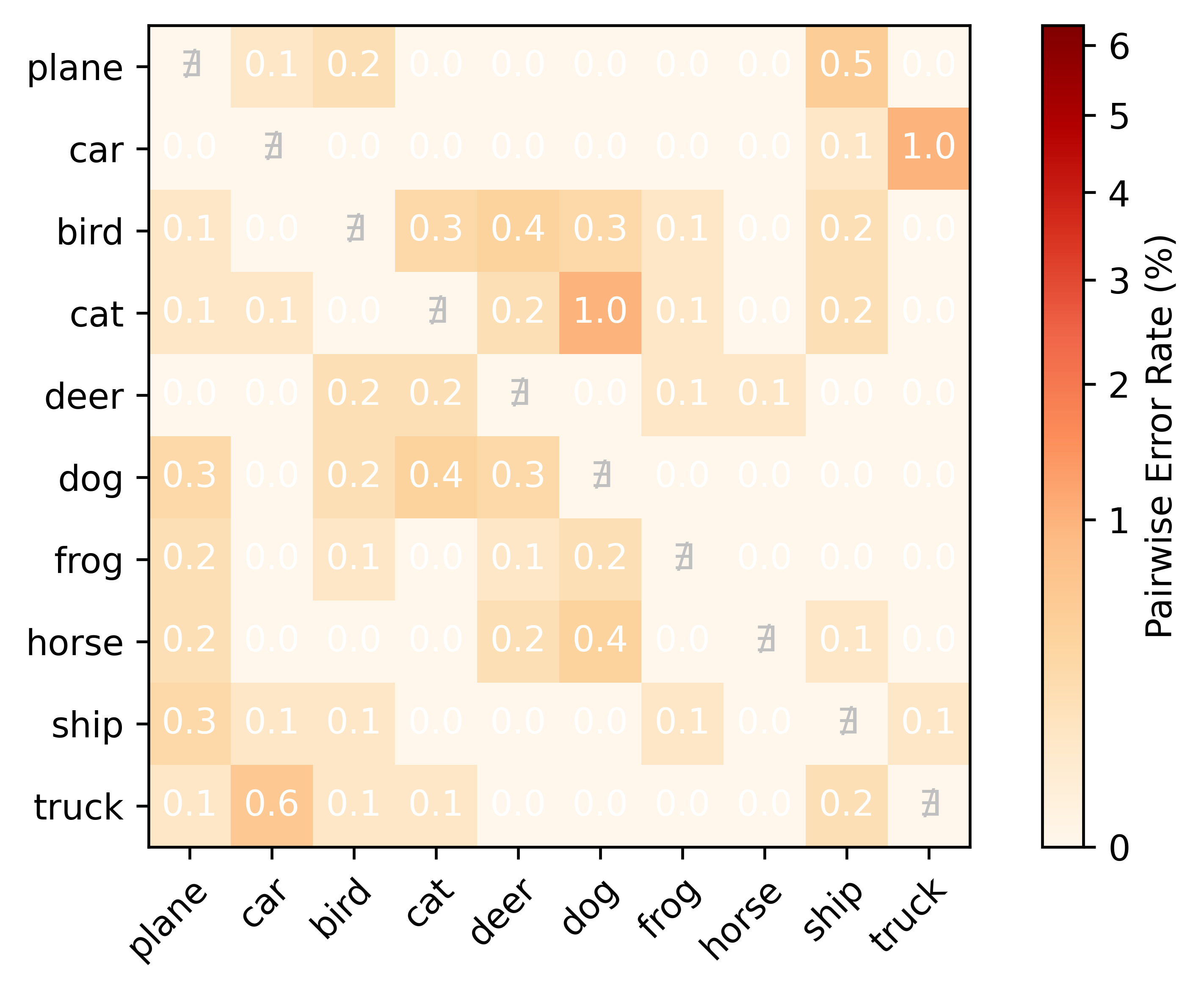}
    \caption{CSADA}
    \label{fig:CIFAR_After}
  \end{subfigure}
    \caption{Comparison of Baseline, Penalty, CSADA Models on CIFAR-10}
  \label{fig:CIFAR}
\end{figure*}
Here we revisit the example in Sec. \ref{sec:example} using our approach. We present results on both CIFAR-10 and MNIST. We train a ResNet-34 \cite{he2016deep} DNN for 300 epochs under cross-entropy loss using stochastic gradient descent (SGD) with momentum $0.9$ and weight decay $10^{-4}$. The learning rate for SGD starts with $0.1$ and decays by $0.33$ every $50$ epochs. We also incorporate standard data augmentation practices when learning these datasets as in \cite{he2016deep}. We denote this model as the Baseline model as we only aim to minimize the empirical risk ($\frac{1}{N}\cdot\sum_{i=1}^N~\left[\ell\left(f(\theta,x_i),y_i\right)\right]$). For both datasets, the Baseline model reaches a training loss of less than $0.001$. 

Now starting with the pretrained model and for every possible pair $(y',z')$ between the 10 classes, we train the penalty method in \eqref{eq:extreme} and our adversarial augmentation model. We assume that predicting $y'$ as $z'$ incurs a cost of one and all other costs are zero. All models are trained for $10$ epochs with a learning rate of $\eta_1 = 10^{-4}$. For CSADA, in CIFAR-10, the hyperparameters are $\eta_2=0.001$, $K=5$, $\epsilon=1$, $\tau=1$, $\lambda=10$. In MNIST, the hyperparameters are $\eta_2=0.05$, $K=5$, $\epsilon=10$, $\tau=1$, $\lambda=10$. All experiments in this section are done using the stochastic version of CSADA in Algorithm \ref{alg:stoch_min}.

The results on both MNIST and CIFAR-10 are shown in Figs. \ref{fig:MNIST} and \ref{fig:CIFAR}. In the figures, each entry of the matrix corresponds to the pairwise error rate of the experiment done with this pair being the critical one. From the results, we can obtain two key insights: i) even with an objective that only aims to reduce critical errors, the results did not improve compared to the Baseline. This again asserts the motivation of the paper and the need to rethink cost-sensitive learning in over-parametrized models. ii) we see that our alternative approach via CSADA results in a significant reduction in the pairwise error rate showing the superior performance of our model. 

\subsection{Cost-sensitive Training on CIFAR-10 and MNIST}

Using the same Baseline trained above, we perform sensitivity analysis on our method and compare the results with other benchmarks. We generate a cost matrix $\mathcal{C}$ where $c(y,z) \sim \text{Pareto(1,1.5)}$. We choose a Pareto distribution to emulate real-life situations such as pharmacy medical dispensing, where some mistakes are life-threatening and far more costly than others.

We start with sensitivity analysis on $\lambda$ using CIFAR-10. Starting with the Baseline model, for each $\lambda$, we trained CSADA for $10$ epochs at a fixed learning rate $\eta_1$ of $5\times 10^{-7}$. The temperature $\tau$ is set to be $3$ in this experiment. Notice that a small learning rate was chosen since we are refining the decision boundaries rather than finding a completely different solution from the Baseline. We use the infinity norm on $\delta$, and set $\epsilon=1$, $K=10$, $\eta_2 = 5\times 10^{-4}$, $\tau=3$. The experiment on each $\lambda$ is replicated for $3$ times to offset randomness. 

Testing results are shown in Fig. \ref{fig:Cost_vs_lmd} where we report the weighted error rate (WER). The weighted error rate finds the average cost per prediction and is given as:
\[\frac{1}{N_{test}}\cdot \sum_{i} c\left(y_i,\argmax_{z\in {\cal Y}}p_z(\theta; x_i)\right).\]
\noindent Notice that when costs are all equal to $1$, WER recovers the overall accuracy.

Fig. \ref{fig:Cost_vs_lmd} provides interesting insights. First, it is important to notice that when $\lambda=0$, we recover the Baseline model, which does not induce any cost-sensitivity. Second, as $\lambda$ increases, the cost first decreases but then increases beyond the Baseline model. This is intuitively understandable since the focus is mainly shifted to adversarial samples when $\lambda$ is very large. This can prevent the model from learning important representations on the natural dataset. Finally, and most importantly, it can be seen from the figure that our approach can indeed significantly reduce the cost when an appropriate $\lambda$ is chosen.  

\begin{figure}[H]
  \centering
  \includegraphics[width=0.65\linewidth]{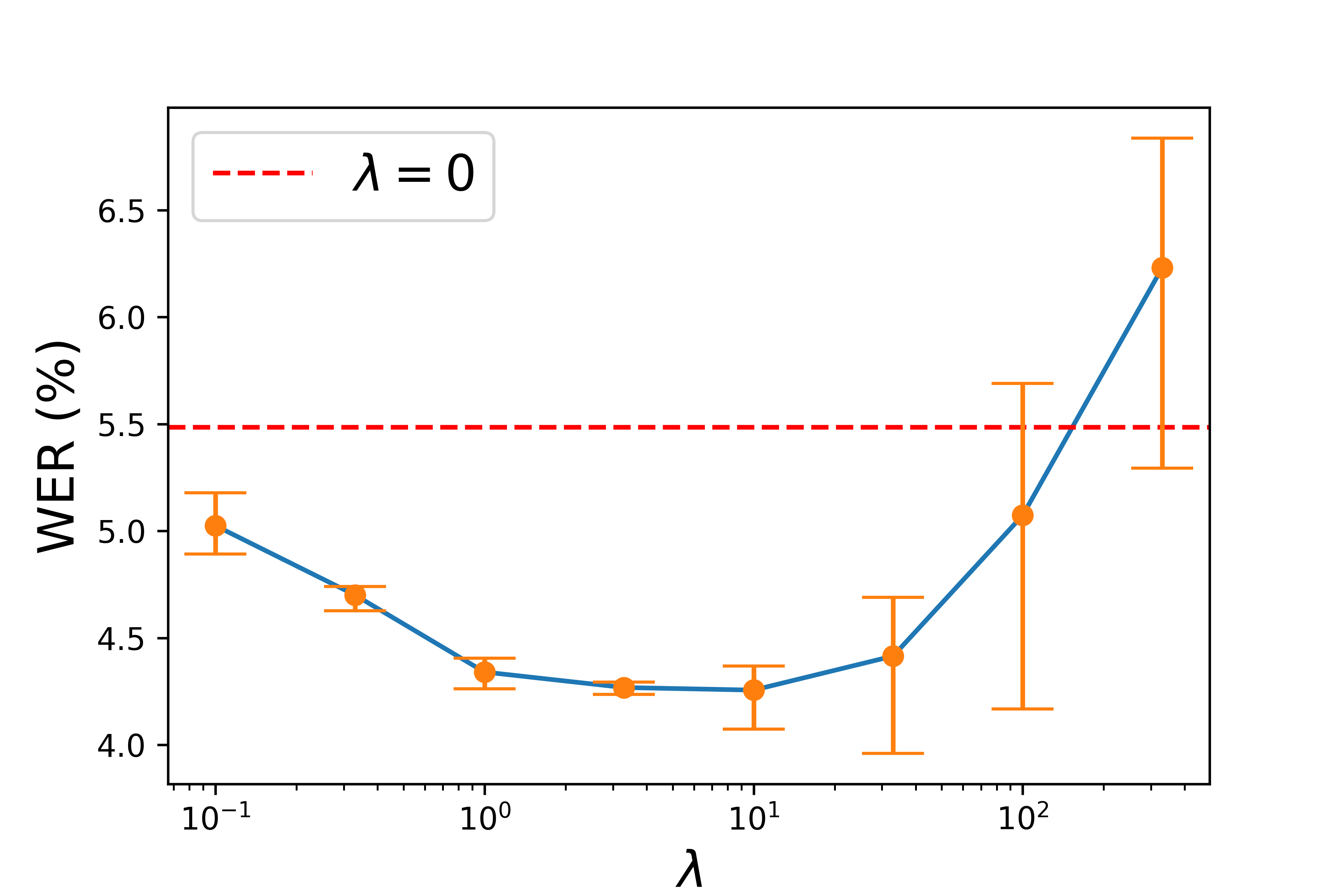}
  \caption{Cost Changes in Response to Hyperparameter $\lambda$}
  \label{fig:Cost_vs_lmd}
\end{figure}

We also test model convergence on CIFAR-10 in Fig.  \ref{fig:CIFAR_Converge}. Based on Fig. \ref{fig:Cost_vs_lmd} we set $\lambda=2$. Similarly, the model is run for $10$ epochs, and we replicate the experiment $10$ times. In Fig. \ref{fig:CIFAR_Converge} we report the mean and 90\% prediction intervals for our loss function ($\ell_{stochastic}$) and the WER. The results again show a significant decrease in both loss and cost at the end of training. 

\begin{figure}[H]
  \centering
  \begin{subfigure}[b]{0.49\linewidth}
    \includegraphics[width=\linewidth]{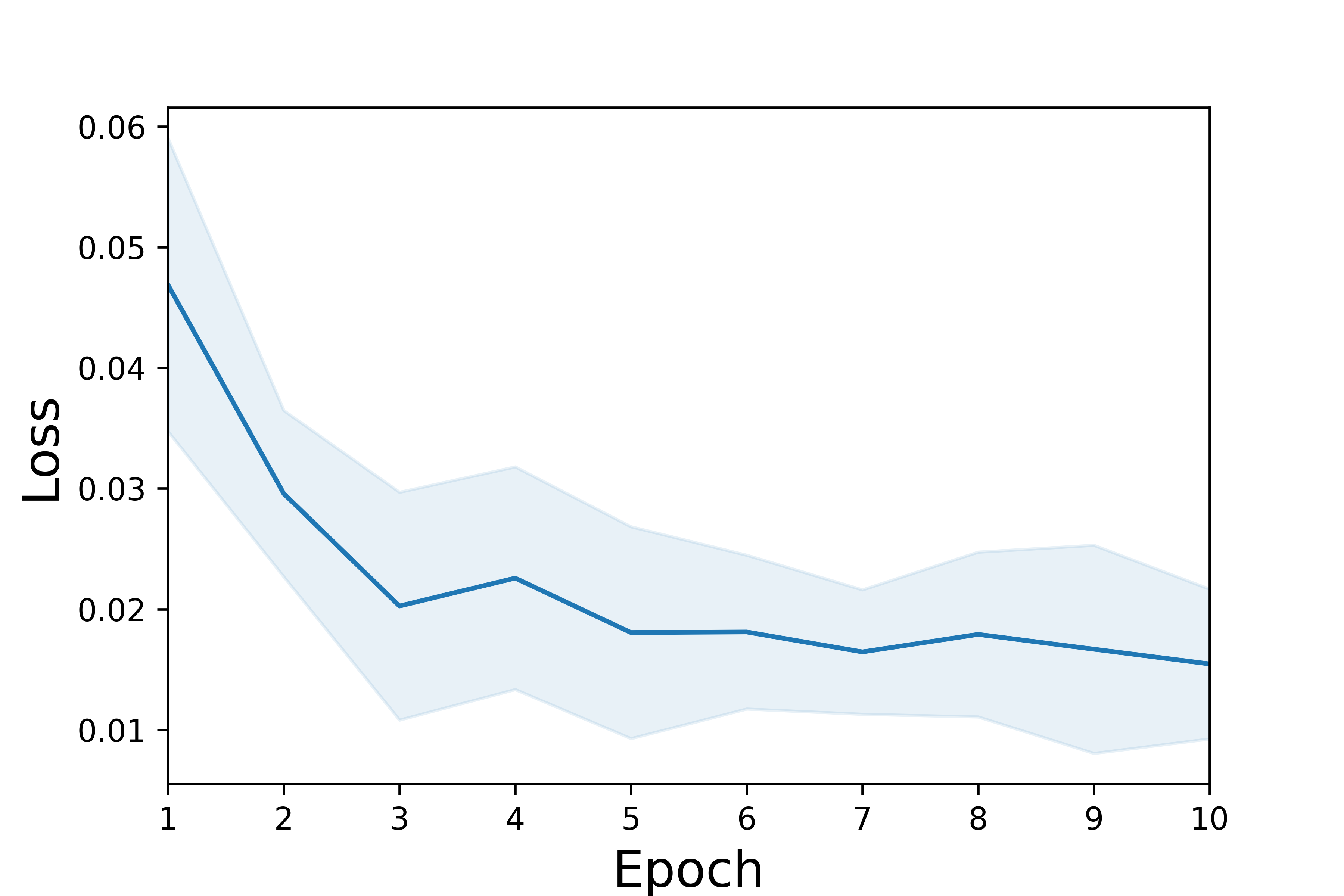}
    \caption{Loss (Training)}
  \end{subfigure}
  \begin{subfigure}[b]{0.49\linewidth}
    \includegraphics[width=\linewidth]{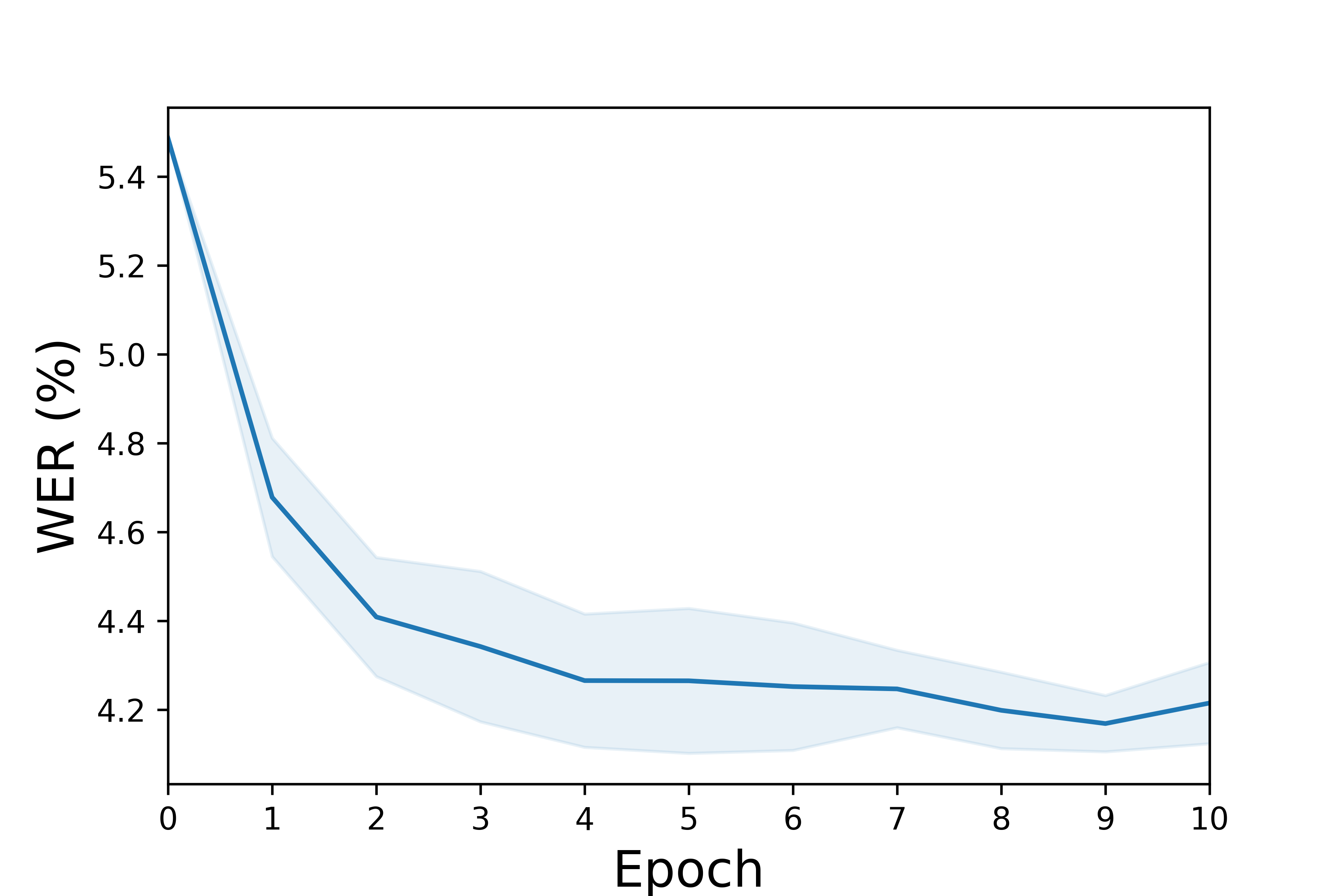}
    \caption{WER (Testing)}
  \end{subfigure}
  \caption{Convergence of Algorithm \ref{alg:stoch_min} on CIFAR-10}
  \label{fig:CIFAR_Converge}
\end{figure}

Next we compare our approach with two benchmarks: (i) $\text{CNN}_{\text{SOSR}}$ proposed in \cite{chung2015cost} and (ii) Adjusted Penalty method (AP) which uses misclassification cost regularization term to penalize critical errors. More specifically the loss function is given as:
\begin{align*}
\min_{\theta} -\frac{1}{N} \cdot \sum_{i=1}^N \Big(&\log(p_{y_i}(\theta,x_i)) \\
&+ \alpha \cdot \sum_{z\in \mathcal{Y}} c(y_i,z) \cdot  \log(1-p_z(\theta,x_i))\Big)
\end{align*}

For all methods, we use ResNet-34 and start from the pre-trained Baseline. Note that for $\text{CNN}_{\text{SOSR}}$ we add a required regression layer. Settings for our model are kept the same. For the benchmarks, we found more epochs are needed for better performance. As such, on the two benchmarks, we use a starting learning rate of $10^{-4}$ that decays by $0.33$ every $10$ epochs. $\alpha$ is set to 5. We followed the method in \cite{chung2015cost} by pretraining only on the cross-entropy loss, so there are no tuning hyperparameters in $\text{CNN}_{\text{SOSR}}$. The mean and standard deviation of WER are shown in Table \ref{table:CIFAR_Comparison}. On CIFAR-10, it can be seen that our algorithm has the lowest cost among the three. Overall, our algorithm requires fewer training epochs, has a lower cost, and does not require additional layers in NNs.

\begin{table}[htbp]
\centering
\resizebox{\columnwidth}{!}{%
\begin{tabular}{ccccc}
\hline
                               & Baseline & CSADA                 & AP     & SOSR        \\ \hline
WER (\%) & 5.49       & \textbf{4.20(0.03)} & 5.07(2.06)  & 5.12(2.08)  \\
Top 1 Cost Pair Error (\%)   & 1.50       & 0.83(0.08)          & 1.20(0.07)  & 1.14(0.16)  \\
Top 1 Cost Pair Error (\%)   & 1.20       & 0.81(0.08)          & 1.06(0.08)  & 1.10(0.29)  \\
Top 1 Cost Pair Error (\%)   & 6.00       & 3.21(0.16)          & 5.08(0.18)  & 5.04(0.91)  \\
Overall   Accuracy (\%)        & 95.70      & 95.54(0.05)         & 95.58(0.15) & 95.53(0.10) \\ \hline
\end{tabular}%
}
\caption{Comparison of Methods on CIFAR-10 Dataset}
\label{table:CIFAR_Comparison}
\end{table}

A similar comparison is further made for the MNIST dataset. Each method was replicated $5$ times. Here, the $\epsilon$ is relaxed to $10$, $K=10$ and $\eta_2=0.05$. Other settings remain the same as in the CIFAR-10. As shown in Table \ref{table:MNIST_Comparison}, competing methods can only slightly reduce costs. In contrast, our approach was able to reduce cost while maintaining similar accuracy to the Baseline model.

\begin{table}[H]
\centering
\resizebox{\columnwidth}{!}{%
\begin{tabular}{ccccc}
\hline
                           & Baseline & CSADA                  & AP         & SOSR            \\ \hline
WER (\%)                   & 0.46   & \textbf{0.25(0.01)} & 0.44(0.01)  & 0.43(0.06)  \\
Top 1 Cost Pair Error (\%) & 0.10   & 0.00(0.00)          & 0.10(0.00)  & 0.10(0.00)  \\
Top 2 Cost Pair Error (\%) & 0.59   & 0.10(0.00)          & 0.54(0.05)  & 0.50(0.15)  \\
Top 3 Cost Pair Error (\%) & 0.00   & 0.00(0.00)          & 0.00(0.00)  & 0.00(0.00)  \\
Overall Accuracy (\%)      & 99.67  & 99.62(0.01)         & 99.66(0.01) & 99.64(0.02) \\ \hline
\end{tabular}%
}
\caption{Comparison of Methods on MNIST Dataset}
\label{table:MNIST_Comparison}
\end{table}

\subsection{Pharmacy Medication Image (PMI) Dataset}

\subsubsection{ Medication Dispensing Errors Overview}
When medications are dispensed, they must match the prescription issued by the physician.  Failing to dispense correct medications can lead to serious medical consequences.  If a computer model is able to recognize the pill inside the medication bottle and confirm it matches the product written on the prescription, dispensing errors can be minimized.  For this case study, the classification model takes pill images as input and predicts the medication product. However, in using a computer model for supporting medication dispensing, the costs can be different across different pairs of  pills. For example, classifying the prescribed Amiodarone Hydrochloride 200 MG Oral Tablet (i.e., a medication used to control a person's heart rate) when the true label is Allopurinol 100 MG Oral Tablet (i.e., a medication used to prevent gout attacks) can result in a patient going untreated for gout or can cause 
toxicity of the lungs \cite{amiodarone}. Other prediction errors, such as confusing one manufacturer for a second manufacturer for an acid reflux medication - Ranitidine 150 MG Oral Tablet - is not likely to result in any harm to the patient.  In our study, we assigned higher costs to critical mistakes. The imbalanced cost nature of medication errors makes this a good application for cost-sensitive training.

\subsubsection{ PMI Dataset Description}
The dataset consists of a collection of pill images and a meta-table reporting their National Drug Codes (NDC). NDCs are unique product identifiers that are used to distinguish different medication products on the basis of ingredient, strength, dose form, and manufacturer.  The physical features of the pill for each NDC are a distinct  combination of shape, size, color, scoring, and imprint.  These pill images show the inside of filled medication bottles from a top-down view (see Fig. \ref{NLM_Image_Examples}).  The dataset includes 13955 images from 20 distinct NDCs. A list of the NDCs along with their key physical features is presented in Table \ref{table:NLM20_summary} in the Appendix. Sample sizes in different NDCs are imbalanced. For each NDC, images are randomly split into training/validation/testing sets at a ratio of approximately 6:2:2. The size of the input images is $1024\times 960$ and was center-cropped to $960\times 960$. The dataset was made publicly available at \url{https://deepblue.lib.umich.edu/data/concern/data_sets/6d56zw997} for reproducibility purposes and to encourage further research along this line. 

\begin{figure}[htbp!]
  \centering
  \begin{subfigure}[b]{0.24\linewidth}
    \includegraphics[width=\linewidth]{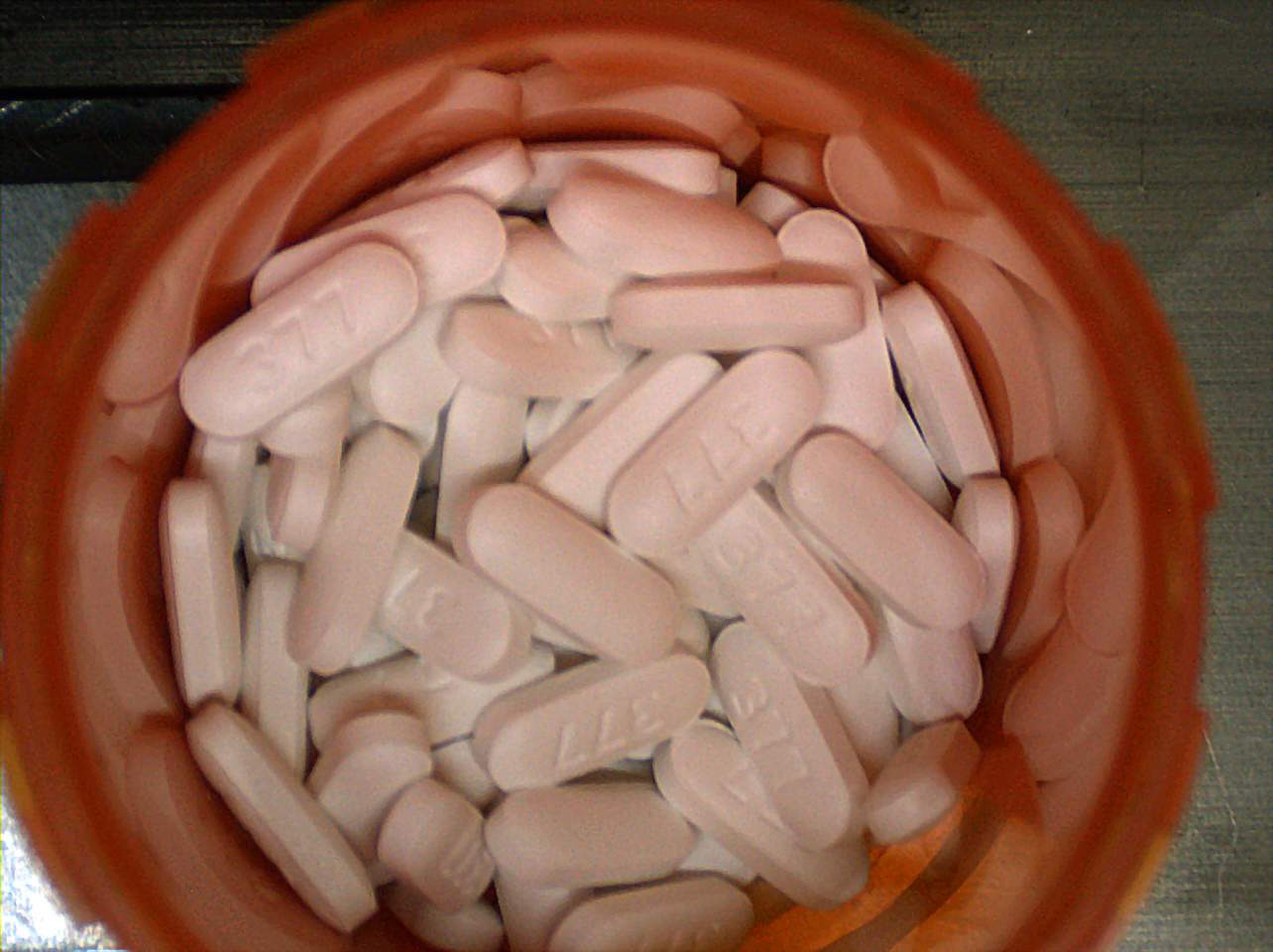}
    \caption{57664-0377}
  \end{subfigure}
  \begin{subfigure}[b]{0.24\linewidth}
    \includegraphics[width=\linewidth]{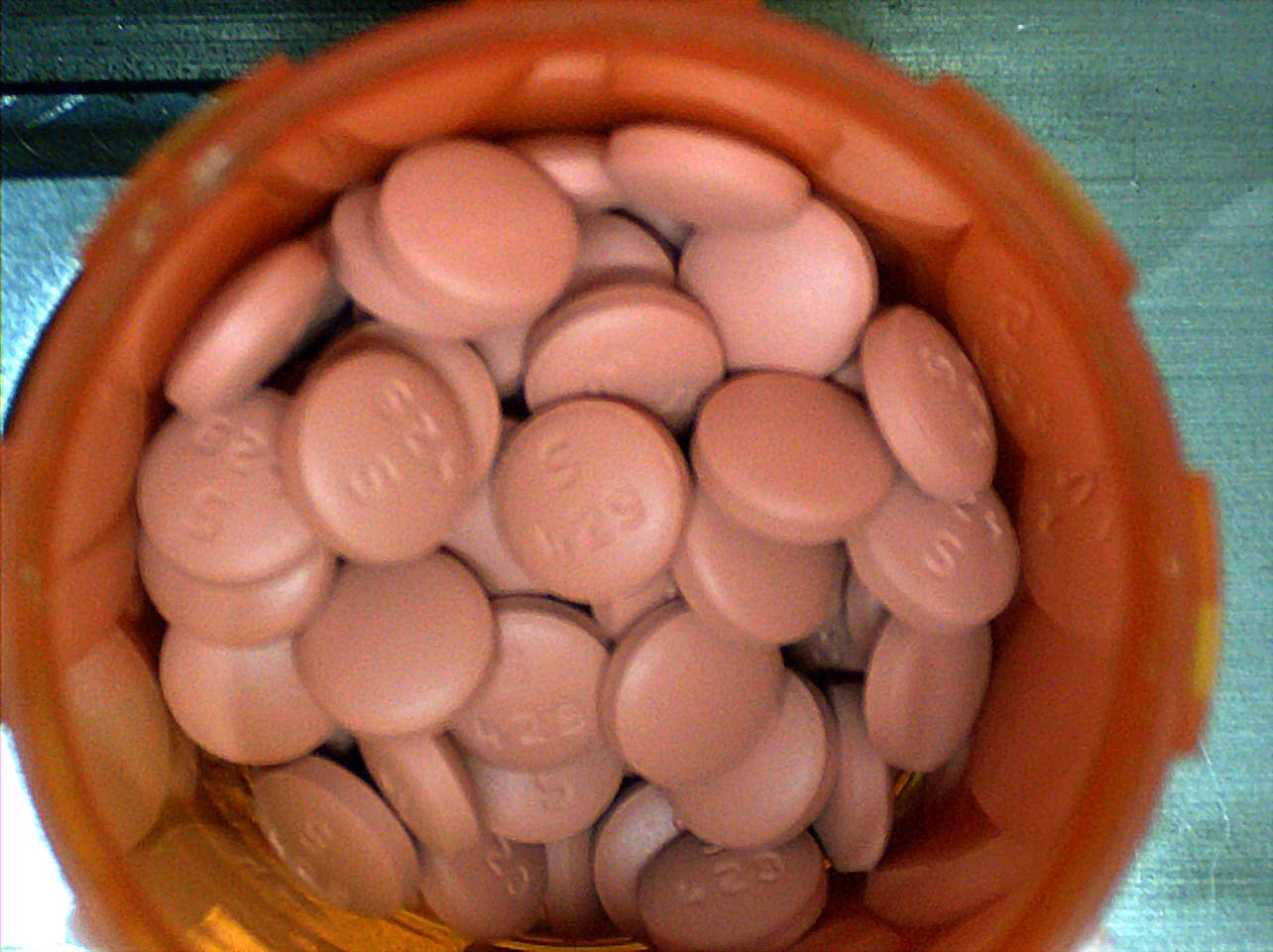}
    \caption{64380-0803}
  \end{subfigure}
  \begin{subfigure}[b]{0.24\linewidth}
    \includegraphics[width=\linewidth]{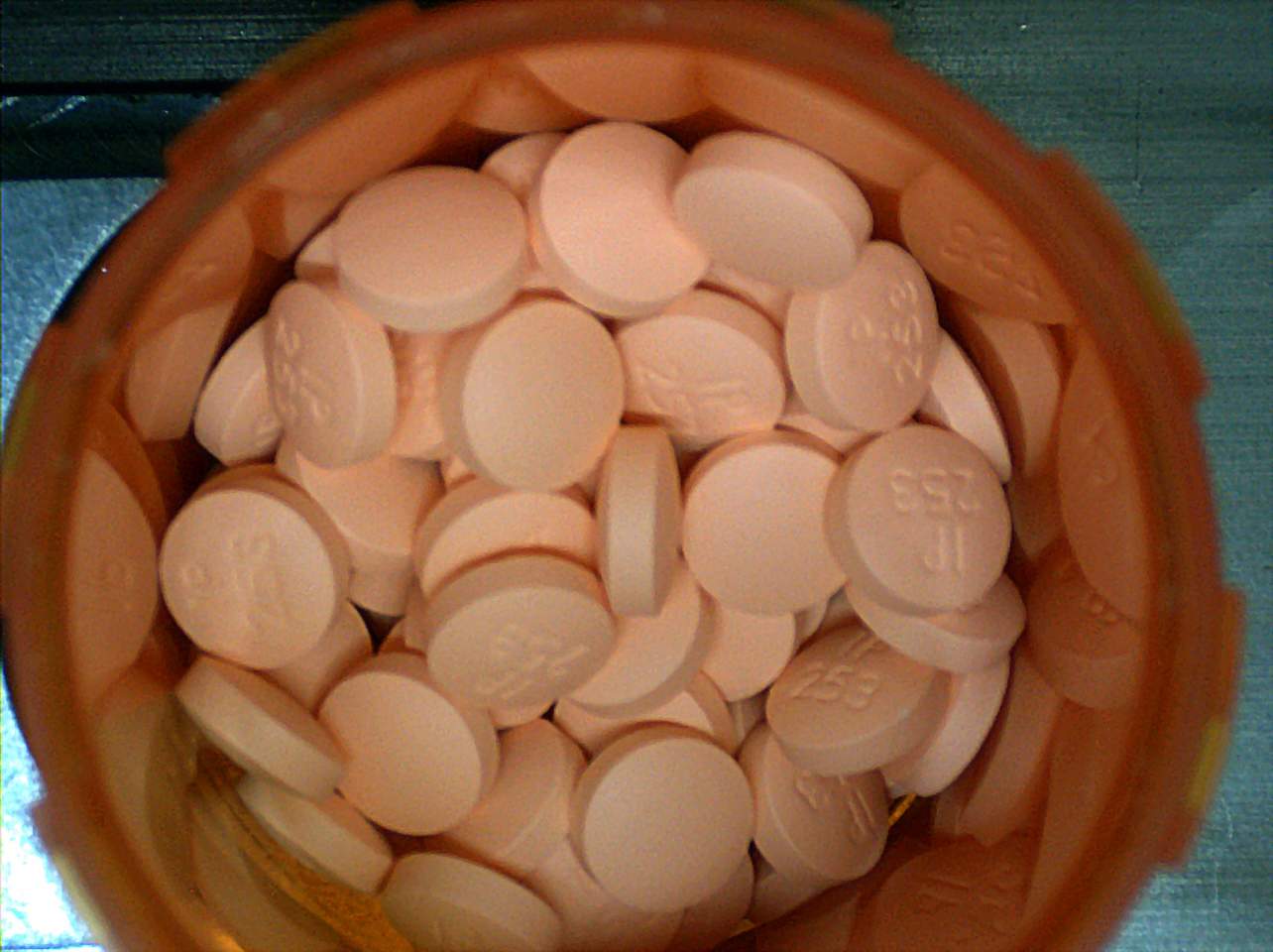}
    \caption{65162-0253}
  \end{subfigure}
    \begin{subfigure}[b]{0.24\linewidth}
    \includegraphics[width=\linewidth]{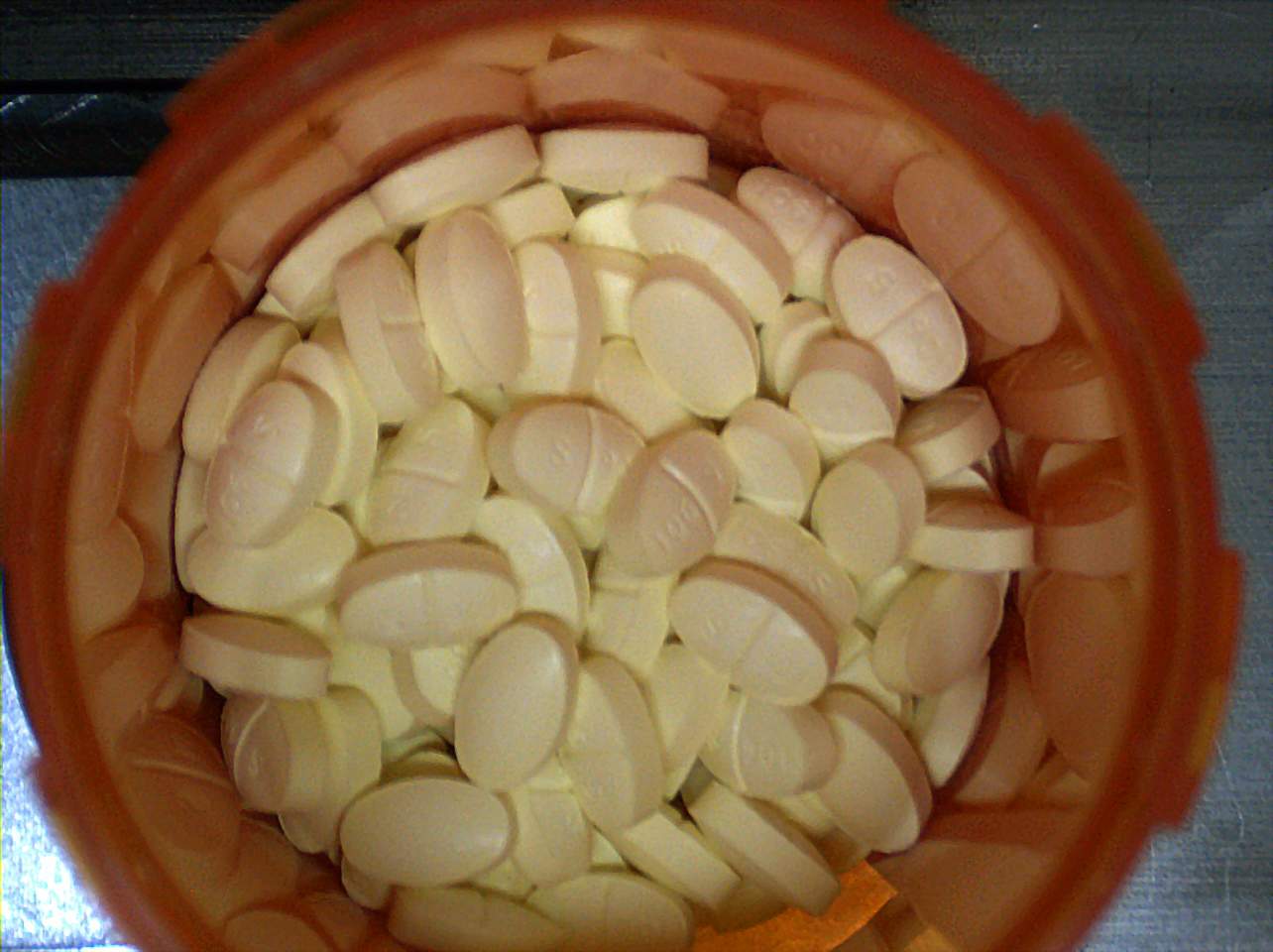}
    \caption{67253-0901}
  \end{subfigure}
  \caption{PMI Dataset Examples}
  \label{NLM_Image_Examples}
\end{figure}

\subsubsection{Baseline Model Training on PMI Dataset}
To train our Baseline model, we take the following steps for data augmentation. Whenever an image is sampled, it is first randomly cropped to a ratio that is generated randomly in the range of $(0.8,1)$. Then the image is resized to $224\times 224$. The resized images are rotated by a degree that is generated randomly from the tuple $(0,90,180,270)$. Finally, images are standardized with calculated mean and variance. We start with a pre-trained ResNet-34 model and train our vanilla Baseline model on the PMI Dataset. We first replace the last layer of the pre-trained model with randomly initialized weights. Freezing all but the last layer, the model is first trained for $10$ epochs with a fixed learning rate of $10^{-4}$. Then, all the layers were defrosted and trained for $20$ epochs with a learning rate of $10^{-5}$ for the first $10$ epochs and $10^{-6}$ for the rest. ADAM optimizer \cite{kingma2015adam} is used in this task. At the end of the training, the vanilla Baseline model achieved 99.61\% accuracy on the test set. However, despite the good overall accuracy, we observed examples of several critical errors in the validation set.

\subsubsection{ Cost-sensitive Training on PMI Dataset}
Expert-costs are assigned to pairs according to their critical levels.  A pharmacist (CL) reviewed the errors from the vanilla Baseline model to identify examples of critical pairs, which were weighted based on their potential danger. Table \ref{table:NLM20_Cost} lists the assigned expert-costs for the four critical pairs, and the remaining non-critical pairs are assigned costs of 1. Starting from the vanilla Baseline model, we trained the model with three different methods and compared their performance with the vanilla Baseline model. To be consistent with the Baseline model, ADAM optimizer is used for all three cost-sensitive methods. Our method (CSADA) was trained for $10$ epochs at a fixed learning rate $\eta_1$ of $1\times10^{-7}$. Other hyperparameters are $\eta_2=1\times10^{-4}$, $K=50$, $\lambda=2$, $\epsilon=1$. The penalty method and $\text{CNN}_{\text{SOSR}}$ were trained for $50$ epochs, and the learning rate starts at $10^{-5}$ and decays by $0.33$ every $10$ epochs. We choose $\alpha=5$ for the AP method. The hyperparameters for the two benchmarks were selected to achieve the best performance on the validation data set. The testing performance is shown in Table \ref{table:NLM20_Compare_34}. Five replications were conducted for each method, where the mean of the five replications is presented along with the standard deviation in the bracket. Here we directly reported the total costs since the mistakes are few. The total cost is simply the summation of expert costs on all misclassifications. We also reported the number of errors on different pairs in Table \ref{table:NLM20_Cost}. The results show that only CSADA was able to cut the cost (by half) and prevent some critical errors. In turn, AP and SOSR exhibited no statistically significant improvement over the Baseline, especially given the high variance. Indeed, this is not surprising as the training loss was about $0.003$. This again confirms that in over-parametrized models where model capacity is not a concern, cost-sensitive training can be very challenging if we completely depend on a training set that can be perfectly fit. 

\begin{table}[htbp]
\centering
\resizebox{\columnwidth}{!}{%
\begin{tabular}{ccccc}
\hline
                                          & Baseline & CSADA              & AP          & SOSR        \\ \hline
Total Cost                                      & 36.0       & \textbf{18.0(1.0)} & 31.0(7.8)  & 29.2(8.8)  \\
50111-0434 to 00591-0461 Error Rate (\%)  & 0.50      & 0.00(0.00)         & 0.10(0.22)  & 0.30(0.27)  \\
53489-0156 to 68382-0227 Error Rate (\%)  & 0.00     & 0.00(0.00)         & 0.00(0.00)  & 0.10(0.22)  \\
53746-0544 to 00378-0208 Error Rate (\%)  & 3.12     & 3.12(0.00)         & 3.12(0.00)  & 3.12(0.00)  \\
68382-0227 to 53489-0156 Error Rate (\%) & 2.56     & 0.00(0.00)         & 2.56(1.81)  & 0.00(0.00)  \\
Overall Accuracy (\%)                     & 99.61    & 99.68(0.04)        & 99.53(0.09) & 99.53(0.05) \\ \hline
\end{tabular}
}
\caption{Comparison Across Benchamrks (ResNet-34)}
\label{table:NLM20_Compare_34}
\end{table}

\begin{table}[htbp]
\centering
\resizebox{\columnwidth}{!}{%
\begin{tabular}{cccccc}
\hline
Error Type                      & Expert Cost  & Baseline & CSADA    & AP         & SOSR      \\ \hline
50111-0434 to 00591-0461  & 10                    & 1        & 0.0(0.0) & 0.2(0.4)   & 0.6(0.5)  \\
53489-0156 to 68382-0227  & 10                    & 0        & 0.0(0.0) & 0.0(0.0)   & 0.2(0.4)  \\
53746-0544 to 00378-0208  & 10                    & 1        & 1.0(0.0) & 1.0(0.0)   & 1.0(0.0)  \\
68382-0227 to 53489-0156  & 8                     & 1        & 0.0(0.0) & 1.0(0.7)   & 0.0(0.0)  \\
Non-critical Pairs              & 1                     & 8        & 8.0(1.0) & 11(2.3) & 11.2(0.8) \\ \hline
\end{tabular}
}
\caption{Average Number of Errors Across 5 Replications}
\label{table:NLM20_Cost}
\end{table}

To further confirm the challenges in over-parameterization, we compared the three models on a shallower network (ResNet-9). We trained the Baseline model under cross-entropy loss using ADAM optimizer for $300$ epochs. The learning rate started with $0.1$ and decayed every $50$ epochs by $0.33$. Using the Baseline model, CSADA was trained for additional $10$ epochs with a fixed learning rate $\eta_1$ of $5\times 10^{-5}$. Other hyperparameters follow $\eta_2=0.01$, $K=10$, $\lambda=10$, $\epsilon=1$, and $\tau=3$. Both AP and SOSR were trained for $50$ epochs starting from the Baseline. The learning rate started at $0.001$ and decayed every $10$ epochs by $0.33$. Hyperparameter $\alpha$ in AP was set to $5$. 

As shown in Table \ref{table:NLM20_Compare_9} the testing accuracy is now much smaller since we are using a model with less capacity. Also, the training error of Baseline now increased to around 0.09. While AP and SOSR reduced the cost, our model still shows superior performance. Interestingly, both methods AP and SOSR performed better on the smaller, less-complex network compared to ResNet-34. This further emphasizes our claim that existing cost-aware approaches encounter degraded performance when over-parameterized networks are used. 


\begin{table}[htbp]
\centering
\resizebox{\columnwidth}{!}{%
\begin{tabular}{ccccc}
\hline
                                          & Baseline & CSADA              & AP          & SOSR        \\ \hline
Total Cost                                      & 434.0       & \textbf{345.4(14.2)} & 381.2(16.6)  & 413.6(23.8)  \\
50111-0434 to 00591-0461 Error Rate (\%)  & 4.50      & 1.70(0.84)         & 2.00(0.79)  & 1.40(0.42)  \\
53489-0156 to 68382-0227 Error Rate (\%)  & 1.50     & 1.70(0.45)         & 1.70(0.27)  & 1.20(0.91)  \\
53746-0544 to 00378-0208 Error Rate (\%)  & 18.75     & 15.12(2.83)         & 15.62(2.21)  & 13.12(1.40)  \\
68382-0227 to 53489-0156 Error Rate (\%) & 25.64     & 21.54(2.92)         & 21.54(3.44)  & 30.77(11.47)  \\
Overall Accuracy (\%)                     & 92.75    & 92.84(0.17)        & 92.43(0.17) & 91.21(0.35) \\ \hline
\end{tabular}
}
\caption{Comparison of Methods (ResNet-9)}
\label{table:NLM20_Compare_9}
\end{table}

\section{Conclusion}
\label{sec:conclusion}
In this paper, we study the problem of cost-sensitive classification that arises in applications where different misclassifications errors have different costs. For this problem, we propose a data augmentation cost-sensitive method that generates various targeted adversarial examples that are used in training to push decision boundaries in directions that minimize critical errors. Mathematically, we propose a penalized cost-aware bi-level optimization framework that penalizes the loss incurred by the generated adversarial examples. We further propose a multi ascent descent gradient-based algorithm for solving the optimization problem. The empirical performance of our model is demonstrated through various experiments on MNIST, CIFAR-10, and a pharmacy medication image (PMI) dataset, which we made publicly available. In all experiments, our method effectively minimized the overall cost
and reduced critical errors while achieving comparable performance in terms of overall accuracy. While we motivate the significance of our model in over-parameterized DNNs, our framework can be easily applied to other machine learning models.

\appendix[PMI Dataset Metadata]
\label{append_summary}
In this section, we provide aThis appendix provides a reference table characterizing the key physical features of each medication (see Table \ref{table:NLM20_summary}). The \textit{Imprint} column lists all the imprinted labels on medication pills. The size column reports the width of pills in millimeters. Within an NDC, images are randomly split into training/validation/testing sets at a ratio of approximately 6:2:2. Classes are imbalanced due to their different frequencies during data collection.
\begin{table*}[tb]
\centering
\resizebox{\linewidth}{!}{%
\begin{tabular}{llccccccc}
\hline
NDC        & Drug Name                                  & Shape   & Color  & Imprint    & Size & Training Size & Validation Size & Testing Size \\ \hline
00378-0208 & Furosemide 20 MG Oral Tablet               & Round   & White  & M2         & 6    & 602           & 199             & 200          \\
00378-3855 & Escitalopram 5 MG Oral Tablet              & Round   & White  & M;EC5      & 6    & 129           & 42              & 42           \\
00591-0461 & Glipizide 10 MG Oral Tablet                & Round   & White  & Watson;461 & 10   & 148           & 48              & 48           \\
16729-0020 & pioglitazone 15 MG Oral Tablet             & Round   & White  & P;15       & 5    & 258           & 84              & 85           \\
16729-0168 & Escitalopram 5 MG Oral Tablet              & Round   & White  & 5          & 5    & 129           & 42              & 42           \\
50111-0434 & Trazodone Hydrochloride 100 MG Oral Tablet            & Round & White  & PLIVA;434 & 11 & 601 & 199 & 200 \\
53489-0156 & Allopurinol 100 MG Oral Tablet             & Round   & White  & MP;71      & 10   & 600           & 200             & 200          \\
53746-0544 & Primidone 50 MG Oral Tablet                & Round   & White  & AN;44      & 6    & 97            & 31              & 32           \\
57664-0377 & tramadol hydrochloride 50 MG Oral Tablet   & Capsule & White  & 377        & 13   & 601           & 199             & 200          \\
62037-0831 & 24 HR metoprolol succinate 50 MG Extended  & Round   & White  & 831        & 9    & 600           & 200             & 200          \\
62037-0832 & 24 HR metoprolol succinate 100 MG Extended & Round   & White  & 832        & 10   & 601           & 200             & 200          \\
64380-0803 & Ranitidine 150 MG Oral Tablet (Strides Pharma         & Round & Brown  & S;429     & 10 & 548 & 181 & 182 \\
65162-0253 & Ranitidine 150 MG Oral Tablet (Amneal Pharmaceuticals & Round & Orange & IP;253    & 9  & 557 & 184 & 185 \\
67253-0901 & Alprazolam 0.5 MG Oral Tablet              & Oval    & Yellow & S901       & 9    & 497           & 164             & 165          \\
68382-0008 & lamotrigine 100 MG Oral Tablet             & Round   & White  & ZC;80      & 10   & 423           & 139             & 140          \\
68382-0227 & Amiodarone hydrochloride 200 MG Oral Tablet           & Round & White  & ZE;65     & 10 & 120 & 39  & 39  \\
69097-0127 & Amlodipine 5 MG Oral Tablet                & Round   & White  & 127;C      & 8    & 600           & 200             & 200          \\
69097-0128 & Amlodipine 10 MG Oral Tablet               & Round   & White  & 128;C      & 8    & 600           & 200             & 200          \\
69315-0904 & Lorazepam 0.5 MG Oral Tablet               & Round   & White  & EP;904     & 5    & 357           & 118             & 118          \\
69315-0905 & Lorazepam 1 MG Oral Tablet                 & Round   & White  & EP;905;1   & 7    & 325           & 107             & 108          \\ \hline
\end{tabular}
}
  \caption{Metadata for PMI Dataset}
  \label{table:NLM20_summary}
\end{table*}

\ifCLASSOPTIONcompsoc
  \section*{Acknowledgments}
\else
  \section*{Acknowledgment}
\fi
Research reported in this publication was supported by the National Library of Medicine of the National Institutes of Health in the United States under award number R01LM013624.

\bibliography{references}
\end{document}